\documentclass{article}

\usepackage{microtype}
\usepackage{graphicx}
\usepackage{subfigure}
\usepackage{booktabs} 

\usepackage{hyperref}

\usepackage[accepted]{icml2024}

\usepackage{amsmath}
\usepackage{amssymb}
\usepackage{mathtools}
\usepackage{amsthm}

\usepackage[capitalize,noabbrev]{cleveref}

\theoremstyle{plain}
\newtheorem{theorem}{Theorem}[section]
\newtheorem{proposition}[theorem]{Proposition}
\newtheorem{lemma}[theorem]{Lemma}
\newtheorem{corollary}[theorem]{Corollary}
\theoremstyle{definition}
\newtheorem{definition}[theorem]{Definition}

\theoremstyle{remark}
\newtheorem{remark}[theorem]{Remark}

\usepackage[textsize=tiny]{todonotes}

\usepackage{bbm}

\newcommand{\E}{\mathbb{E}}
\newcommand{\R}{\mathbb{R}}
\newcommand{\eps}{\varepsilon}
\newcommand{\indic}[1]{\mathbbm{1}_{#1}}

\newcommand{\ALG}{\textsf{ALG}}
\newcommand{\A}{\textsf{A}}
\newcommand{\CR}{\textsf{CR}}
\newcommand{\ratio}{\textsf{R}}
\newcommand{\OPT}{\textsf{OPT}}
\newcommand{\RR}{\textsf{RR}}
\newcommand{\RTC}{\textsf{RTC}}

\newcommand{\switch}{\textsf{Switch}}
\newcommand{\CRRR}{\textsf{CRRR}}

\newcommand{\spt}{S}

\renewcommand{\phi}{\varphi}
\newcommand{\uij}{u}
\newcommand{\V}{\mathcal{V}}

\icmltitlerunning{Non-clairvoyant Scheduling with Partial Predictions}

\begin{document}

\twocolumn[
\icmltitle{Non-clairvoyant Scheduling with Partial Predictions}



\icmlsetsymbol{equal}{*}

\begin{icmlauthorlist}
\icmlauthor{Ziyad Benomar}{fairplay,X}
\icmlauthor{Vianney Perchet}{fairplay,criteo}
\end{icmlauthorlist}

\icmlaffiliation{fairplay}{ENSAE, FAIRPLAY joint team, CREST, Palaiseau, France}
\icmlaffiliation{criteo}{Criteo AI Lab, FAIRPLAY joint team, Paris, France}
\icmlaffiliation{X}{Ecole polytechnique, Palaiseau, France}

\icmlcorrespondingauthor{Ziyad Benomar}{ziyad.benomar@ensae.fr}

\icmlkeywords{Learning-augmented algorithms, online algorithms, scheduling algorithms with predictions}

\vskip 0.3in
]



\printAffiliationsAndNotice{} 

\begin{abstract}
The non-clairvoyant scheduling problem has gained new interest within learning-augmented algorithms, where the decision-maker is equipped with predictions without any quality guarantees. In practical settings, access to predictions may be reduced to specific instances, due to cost or data limitations. Our investigation focuses on scenarios where predictions for only $B$ job sizes out of $n$ are available to the algorithm. We first establish near-optimal lower bounds and algorithms in the case of perfect predictions. Subsequently, we present a learning-augmented algorithm satisfying the robustness, consistency, and smoothness criteria, and revealing a novel tradeoff between consistency and smoothness inherent in the scenario with a restricted number of predictions.
\end{abstract}

\section{Introduction}

Optimal job scheduling is a longstanding and actively studied class of optimization problems \cite{panwalkar1977survey, lenstra1978complexity, graham1979optimization, martel1982preemptive, cheng1990state, lawler1993sequencing,  pinedo2012scheduling}, with applications in various domains spanning from supply chain management \cite{hall2003supply, ivanov2016dynamic} to operating systems \cite{jensen1985time, ramamritham1994scheduling, steiger2004operating}.
A particular setting is preemptive single-machine scheduling \cite{pinedo2012scheduling, baker2013principles}, where $n$ jobs $i \in [n]$ must be executed on the same machine, with the possibility of interrupting a job and resuming it afterward, and the objective is to minimize the sum of their completion times. An algorithm is called \textit{clairvoyant} if it has initial access to the job sizes, otherwise, it is called \textit{non-clairvoyant} \cite{motwani1994nonclairvoyant}.
The design of non-clairvoyant scheduling algorithms is a classical problem in competitive analysis and online algorithms \cite{borodin2005online}. In this paradigm, decisions must be made in an environment where the parameters governing the outcome are unknown or might evolve over time. 

Due to the inherent difficulty of the problems in competitive analysis, the performance of any online algorithm stays bounded away from that of the optimal offline algorithm. However, the ascent of machine learning motivated the incorporation of predictions in algorithm design, which started in works such as \cite{munoz2017revenue, kraska2018case}, then was formalized in \cite{lykouris2018competitive} and \cite{purohit2018improving}. Since then, \textit{learning-augmented algorithms} became a popular research topic and had multiple applications \cite{mitzenmacher2022algorithms}.
The outcome of these algorithms depends both on the parameters of the problem and the quality of the predictions. They are required to have a performance that is near-optimal when the predictions are accurate (consistency), near the worst-case performance without advice if the predictions are arbitrarily erroneous (robustness), and that degrades smoothly as the prediction error increases (smoothness).

In practice, predictions often incur costs and, at times, are infeasible due to the lack of data. It is, therefore, crucial to understand the limitations and the feasible improvements with a restrained number of predictions in scenarios with multiple unknown variables. This question was first investigated for the caching problem \cite{im2022parsimonious}, and very recently for metrical task systems \cite{sadek2024algorithms}, in settings where the algorithm is allowed to query a limited number of predictions. It was also explored for the scheduling problem \cite{benomar2023advice}, assuming that the decision-maker can query the true sizes of $B$ jobs out of $n$. The authors present a $\big(2-\tfrac{B(B-1)}{n(n-1)}\big)$-competitive algorithm, and they give a lower bound on the competitive ratio of any algorithm only when $B = o(n)$. The case of imperfect predictions, however, is not examined.

In non-clairvoyant scheduling, besides the querying model studied in the works mentioned above, predictions of the sizes of certain jobs $i \in I$ may be available, where $I \subset [n]$, perhaps derived from previous executions of similar tasks. 
Assuming that $I$ is a subset of $[n]$ of size $B$, taken uniformly at random,
we examine the limitations and possible improvements of non-clairvoyant algorithms.


\subsection{Contributions}
We initiate our analysis by addressing the scenario of perfect predictions. We establish a lower bound on the $(n, B)$-competitive ratio of any algorithm (for fixed $n \geq 2$ and $B \leq n$), which extends the lower bound of $2 - \frac{4}{n+3}$ in the non-clairvoyant case \cite{motwani1994nonclairvoyant}. Considering that $B =  wn + o(n)$ for some $w \in [0,1]$, we derive from the prior bound that the competitive ratio of any algorithm is at least $2 - w - (\frac{4}{e}-1)w(1-w)$, and we show an improved bound of $2 - w - (3-2\sqrt{2})w(1-w)$.  Demonstrating these bounds is considerably more challenging than the case $B = 0$, due to the eventual dependency between the actions of the algorithm and the known job sizes.

In the case of perfect predictions, we show that knowing only the relative order of the $B$ job sizes, without knowledge of their values, enables a $(2-\frac{B}{n})$-competitive algorithm, which improves substantially upon the result of \cite{benomar2023advice}. 
We propose a second algorithm leveraging the true sizes of the $B$ jobs, yielding an $(n,B)$-competitive ratio of $(2 - \frac{B}{n} - \frac{2(1 - B/n)}{n+1})$, which is strictly better than the former, although both are asymptotically equivalent.

Subsequently, we adapt the latter algorithm to handle imperfect predictions. While the difficulty in most works on learning-augmented algorithms lies in ensuring robustness and consistency, smoothness in the case of scheduling with limited predictions is also not immediate.
Alongside the typical consistency-robustness tradeoff, our algorithm also exhibits a consistency-smoothness tradeoff. More precisely, governed by two hyperparameters $\lambda, \rho \in [0,1]$, the $(n,B)$-competitive ratio of the algorithm is at most 
$\min(\frac{2}{1-\lambda}, \frac{C}{\lambda} + \frac{S}{\lambda} \frac{n \E[\eta]}{\OPT})$. 
Here, $\E[\eta]$ denotes the total expected prediction error, $\OPT$ is the objective function achieved by the optimal offline algorithm, $\frac{2}{1-\lambda}$ is the algorithm's robustness, $\frac{C}{\lambda} = \frac{1}{\lambda}(2 - \frac{B}{n} + \rho\frac{B}{n}(1 - \frac{B-1}{n-1}))$ its consistency, and $\frac{S}{\lambda} = \frac{1}{\lambda}( \frac{4}{\rho} (1 - \frac{B}{n}) + \frac{B}{n})$ its smoothness factor, characterizing the sensitivity of the bound to $\E[\eta]$. 
Notably, alterations in the parameter $\rho$ yield opposing variations on the consistency and the smoothness factor. Nonetheless, this tradeoff vanishes for $B$ close to $0$ or $n$, and does not appear, for instance, in \cite{purohit2018improving}, \cite{bampis2022scheduling} or \cite{lindermayr2022permutation}.

We illustrate our results for the case of perfect predictions in Figure \ref{fig:results}, comparing them with the competitive ratio proved in \cite{benomar2023advice}. 

\begin{figure}[h!]
\begin{center}
\centerline{\includegraphics[width=0.75\columnwidth]{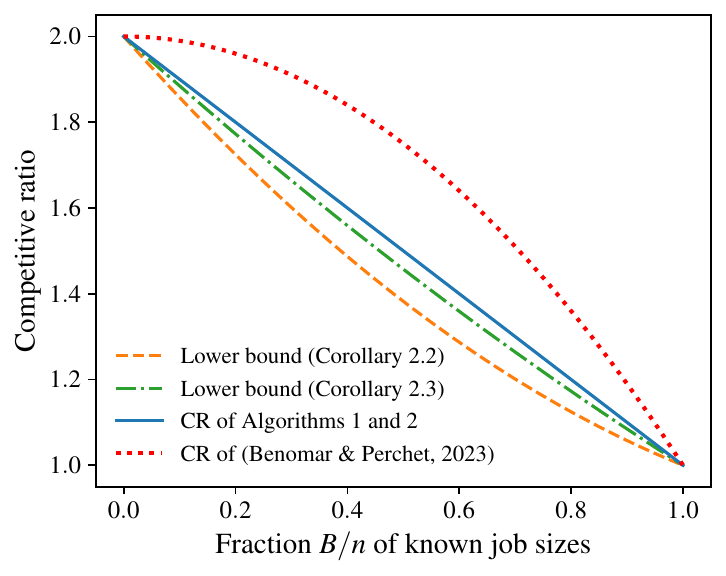}}
\caption{Lower bounds and competitive ratios for $B$ known job sizes.}
\label{fig:results}
\end{center}
\vskip -0.4in
\end{figure}

\subsection{Related work}
Since their introduction in \cite{purohit2018improving, lykouris2018competitive}, learning-augmented algorithms witnessed an exponentially growing interest, as they offered a fresh perspective for revisiting online algorithms, and provided new applications for machine learning in algorithm design \cite{mitzenmacher2022algorithms} and in the implementation of data structures \cite{kraska2018case, lin2022learning}. Many fundamental problems in competitive analysis were studied in this setting, such as ski rental \cite{gollapudi2019online, anand2020customizing, bamas2020primal, diakonikolas2021learning, antoniadis2021learning, maghakian2023applied, shin2023improved}, secretary \cite{antoniadis2020secretary, dutting2021secretaries}, matching \cite{dinitz2021faster, chen2022faster, sakaue2022discrete, jin2022online}, caching and metrical task systems \cite{lykouris2018competitive, chlkedowski2021robust, antoniadis2023online, antoniadis2023paging, christianson2023optimal}. 
In particular, scheduling is one of the problems that were studied most thoroughly. Different works cover various objective functions \cite{purohit2018improving, lattanzi2020online, azar2021flow}, prediction types \cite{antoniadis2021novel, merlis2023preemption, lassota2023minimalistic}, error metrics \cite{im2021non, lindermayr2022permutation}, and other aspects and applications \cite{wei2020optimal, bamas2020learning, dinitz2022algorithms}

The setting of learning-augmented algorithms with limited predictions was initially explored by \citet{im2022parsimonious} for caching. The authors presented an algorithm using parsimonious predictions, with a competitive ratio increasing with the number of allowed queries. In another very recent paper \cite{sadek2024algorithms}, a similar setting is studied for the more general problem of metrical task systems, where the algorithm is allowed to query a reduced number of \textit{action predictions} \cite{antoniadis2023online}, each giving the state of an optimal algorithm at the respective query step. An additional related study by \citet{drygala2023online} focuses on the ski-rental and Bahncard problems in a penalized adaptation of the setting with limited advice, where the cost of the predictions is added to the algorithm's objective function. 
Other works have explored related settings with different types of limited advice. For instance, the setting with a restricted number of perfect hints was examined in the context of online linear optimization by \citep{bhaskara2021logarithmic} and in the multi-color secretary problem by \citep{benomar2023addressing}. Another setting, where the algorithm can query two types of hints—one that is free but possibly inaccurate, and another that is expensive but accurate—has been studied in several problems, such as correlation clustering \citep{silwal2023kwikbucks}, computing minimum spanning trees in a metric space \citep{bateni2023metric}, sorting \cite{bai2024sorting}, and matroid optimization \citep{eberle2024accelerating}.

In the context of scheduling, \citet{benomar2023advice} introduced the $B$-clairvoyant scheduling problem, where an algorithm can query the exact sizes of $B$ jobs at any moment during its execution. They show that the optimal strategy involves querying the sizes of $B$ jobs selected uniformly at random at the beginning of the process. They establish that, if $B = o(n)$, then the competitive ratio of any algorithm is at least $2$, then they provide a $\big(2-\frac{B(B-1)}{n(n-1)}\big)$-competitive algorithm. The same paper also addresses the secretary problem with restricted access to binary predictions of known accuracy, and the ski-rental problem with access to an oracle whose accuracy improves progressively over time. 

The limit scenario $B = 0$ corresponds to the non-clairvoyant scheduling problem, studied in-depth in \cite{motwani1994nonclairvoyant}. In particular, the paper demonstrates that the competitive ratio of any non-clairvoyant algorithm is at least $2$, and that it is achieved by the \textit{round-robin} algorithm, executing all unfinished jobs concurrently at equal rates. On the other hand, $B=n$ corresponds to the setting presented in \cite{purohit2018improving}, where the authors introduce, for all $\lambda \in (0,1)$, a \textit{preferential round-robin} algorithm with robustness $\frac{2}{1-\lambda}$ and consistency $\frac{1}{\lambda}$.

\subsection{Problem and notations}
The decision-maker is given $n$ jobs $i \in [n]$ with unknown sizes $x_1,\ldots,x_n$ to schedule on a single machine, and predictions $(y_i)_{i \in I}$ of $(x_i)_{i \in I}$, with $I$ a uniformly random subset of $[n]$ of size $B$. The objective is to leverage the available predictions to minimize the sum of the completion times. We assume that preemption is allowed, i.e. it is possible to interrupt the execution of a job and resume it later, which is equivalent, by neglecting the preemption cost, to assuming that the jobs can be run in parallel at rates that sum to at most $1$.

To simplify the presentation, we consider that there are predictions $y_1,\ldots,y_n$ of $x_1,\ldots,x_n$, but the decision-maker has only access to $y_{\sigma(1)}, \ldots, y_{\sigma(B)}$, where $\sigma$ is a uniformly random permutation of $[n]$. We denote by $\eta_i = |x_i - y_i|$ the error of the prediction $y_i$, 
and by $\eta^\sigma = \sum_{i=1}^B \eta_{\sigma(i)}$ the total error of the predictions accessed by the algorithm. 

Consider an algorithm $\A$ and an instance $x = (x_1,\ldots,x_n)$ of job sizes, we denote by $\A(x)$ the sum of the completion times of all the jobs when executed by $\A$.
Furthermore, for all $i \neq j \in [n]$ and $t>0$, we denote by 
\begin{itemize}
    \item $\spt_i^\A(t)$ the processing time spent on job $i$ until time $t$,
    \item $t_i^\A = \inf\{ t \geq 0: \spt_i^\A(t) = x_i\}$ its completion time,
    \item $D^\A_{ij} = \spt_i^\A(t^\A_j)$ the total time spent on job $i$ before job $j$ terminates,
    \item and $P^\A_{ij} = D^\A_{ij} + D^\A_{ji}$ the mutual delay caused by $i,j$ to each other.
\end{itemize}
When there is no ambiguity, we omit writing the dependency to $\A$.
With these notations, it holds that $t^\A_i = x_i + \sum_{j \neq i} D^\A_{ji}$ for all $i \in [n]$. Consequently, the objective function of $\A$ can be expressed as
\begin{equation}\label{eq:generic-output}
\A(x) = \sum_{i=1}^n x_i + \sum_{1 \leq i < j \leq n} P^\A_{ij} \;.    
\end{equation}
Observing that, for all $i \neq j \in [n]$, if $i$ terminates before $j$ then $P^\A_{ij} \geq x_i$, otherwise $P^\A_{ij} \geq x_j$, we deduce that $P^\A_{ij} \geq \min(x_i, x_j)$. Equality is achieved by the clairvoyant algorithm that runs the jobs until completion in non-decreasing size order \cite{motwani1994nonclairvoyant}, which is the optimal offline algorithm, that we denote $\OPT$, satisfying
\begin{equation}\label{eq:opt-output}
\OPT(x)
= \sum_{i=1}^n x_i + \sum_{1 \leq i < j \leq n} \min(x_i,x_j)\;.
\end{equation}

When the predictions are perfect, for all $n \geq 2$ and $B \leq n$, we define the $(n,B)$-competitive ratio of algorithm $\A$ as the worst-case ratio between its objective, knowing the sizes of $B$ jobs taken uniformly at random, and that of $\OPT$, on instances of $n$ jobs
\begin{equation*}
    \ratio_{n,B}(\A) = \sup_{x\in (0,\infty)^n} \frac{\E[\A(x)]}{\OPT(x)}\;,
\end{equation*}
where the expectation $\E[\A(x)]$ is taken over the permutation $\sigma$ and the actions of $\A$ if it a randomized algorithm. 

If the number of predictions depends on the number of jobs, i.e.  $B = (B_n)_{n \geq 1}$ defines a sequence of integers, then the competitive ratio of $\A$ is defined by
\begin{equation*}
\CR_B(\A) = \sup_{n \geq 2} \ratio_{n,B_n}(\A)\;.
\end{equation*}
When the predictions are imperfect, the competitive ratio becomes also a function of $\E[\eta^\sigma]$.

\section{Lower Bounds}\label{sec:lowerbounds}
In this section, we assume that the predictions are error-free, and we establish lower bounds on the $(n,B)$-competitive ratio of any algorithm, followed by a lower bound independent of $n$ when $B_n = w n + o(n)$ for some $w \in [0,1]$.
These lower bounds are obtained by constructing random job size instances $x$ such that, for any deterministic algorithm $\A$, the ratio $\E_{\sigma,x}[\A(x)]/ \E_x[\OPT(x)]$ is above them. The result then extends to randomized algorithms and yields bounds on their $(n,B)$-competitive ratio by using Lemma \ref{lem:yao}, which is a consequence of Yao's principle \cite{yao1977probabilistic}. 


In all this section, we consider i.i.d. job sizes. Therefore, we can assume without loss of generality that the $B$ known job sizes are $x_1,\ldots,x_B$.

In the non-clairvoyant case $B=0$, for any algorithm $\A$, taking i.i.d. exponentially distributed sizes gives, with easy computation, that $\E[P^\A_{ij}] = 1$ (Remark \ref{rmk:uij-nonclairvoyant}), which yields, using Equations \eqref{eq:generic-output} and \eqref{eq:opt-output}, the lower bound $2 - \frac{4}{n+3}$ on the competitive ratio \cite{motwani1994nonclairvoyant}. However, if $B > 0$, the algorithm can act according to the information it has on $(x_i)_{i \in I}$, and the dependence between its actions and these job sizes makes the analysis more sophisticated.

For any positive and continuous function $\phi$, and positive numbers $T \geq x > 0$ we denote
\begin{equation}\label{eq:Gphi}
G_\phi(x,T) = \int_0^{T-x} \frac{dt}{\phi(t)} + \frac{x}{\phi(T-x)}\;.
\end{equation}
We prove in the following theorem a generic lower bound, using job sizes sampled independently from the distribution $\Pr(x_i \leq t) = 1-\frac{\varphi(0)}{\varphi(t)}$.

\begin{theorem}\label{thm:lower-bound}
Let $\phi : [0,\infty) \to [0,\infty)$ be a continuously differentiable and increasing function satisfying $\phi(0) > 0$, $\phi'/\phi$ is non-increasing and  $\int_0^\infty \frac{dt}{\phi(t)^2} < \infty$, and let $\alpha_\phi$ a non-negative constant satisfying
\begin{equation}\label{eq:thm-lb-eq}
\int_0^\infty \left\{ \inf_{T \geq x} G_\phi(x,T) \right\} \frac{\phi'(x)}{\phi(x)^2} dx
\geq \alpha_\phi \int_0^\infty \frac{dt}{\phi(t)^2}\;.    
\end{equation}
If $B = wn + o(n)$ for some $w \in [0,1]$, then it holds for any randomized algorithm $\A$ that
\[
\CR_B(\A) \geq 2 - 2(2-\alpha_\phi) w + (3 - 2\alpha_\phi)w^2\;.
\]
Moreover, if $\int_0^\infty \frac{dt}{\phi(t)} < \infty$, then for all $n \geq 2$ and $B \leq n$ 
\[
\ratio_{n,B}(\A)
\geq C_{\phi,n,B} - \frac{C_{\phi,n,B} - 1}{1+ \frac{n-1}{2} \frac{\int_0^\infty \frac{dt}{\phi(t)^2}}{\int_0^\infty \frac{dt}{\phi(t)}}} \;,
\]
where $C_{\phi,n,B} 
= (2 - \tfrac{B}{n}) - (3 - 2 \alpha_\phi) \tfrac{B}{n}\big(1 - \tfrac{B-1}{n-1} \big)$.
\end{theorem}

To establish this theorem, we analyze i.i.d. job sizes sampled from the distribution $\Pr(x_i \leq t) = 1-\frac{\varphi(0)}{\varphi(t)}$. We derive in Lemma \ref{lem:generic-lb} a lower bound on the mutual delays incurred by these jobs during the run of any algorithm $\A$. This involves solving a functional minimization problem, whose solution is expressed using the function $G_\phi$ defined in \eqref{eq:Gphi}. The left term in Inequality \eqref{eq:thm-lb-eq} is proportional to the obtained lower bound, while the right term is proportional to $\E[\min(x_i,x_j)]$, which is the mutual delay caused in a run of $\OPT$. Finally, using the identity \eqref{eq:generic-output}, this inequality, which relates the mutual delays caused respectively by executing $\A$ and $\OPT$ on the chosen job sizes, can be extended to an inequality involving the objectives of both algorithms, giving a lower bound on the competitive ratio.


If $\int_0^\infty \frac{dt}{\phi(t)} = \infty$, the expectation of the job sizes is infinite. In this case, we consider a truncated distribution with a maximum $a > 0$. After completing the analysis, we derive a lower bound that depends on $a$ and $w$, by considering $B = wn + o(n)$ and $n \to \infty$, then, we conclude by taking the limit $a \to \infty$.

For any value $\alpha_\phi$ in Theorem \ref{thm:lower-bound}, observe that $C_{\phi,n,0} = 2$ and $C_{\phi,n,n} = 1$, which means that the lower bound interpolates properly the non-clairvoyant and clairvoyant settings. The remaining task is to choose an adequate function $\phi$ satisfying the conditions of the theorem with $\alpha_\phi$ as large as possible. We first consider exponentially distributed job sizes, often used to prove lower bounds in scheduling problems. This corresponds to $\phi(t) = e^t$.

\begin{corollary}\label{cor:lb-exp}
For any algorithm $\A$, it holds that
\[
\ratio_{n,B}(\A)
\geq  C_{n,B} - \frac{4(C_{n,B}-1)}{n+3}\;,
\]
with 
$C_{n,B} = 2 - \frac{B}{n} - (\tfrac{4}{e} - 1)\frac{B}{n}\big(1 - \frac{B-1}{n-1}\big)$.
In particular, if $B = wn + o(n)$ then
\[
\CR_B(\A)
\geq (2 - w) - (\tfrac{4}{e} - 1)w(1-w)\;.
\]
\end{corollary}

Corollary \ref{cor:lb-exp} gives, in particular, that $\ratio_{n,0}(\A) \geq 2 - \frac{4}{n+3}$, which corresponds exactly to the tight lower bound for the non-clairvoyant scheduling problem \cite{motwani1994nonclairvoyant}. However, the bound is not tight for all values of $B \leq n$. To refine it, we consider distributions that would be more difficult to process by the algorithm. One idea is to sample a different parameter $\lambda_i \sim \mathcal{E}(1)$ independently for each $i \in [n]$, then sample $x_i \sim \mathcal{E}(\lambda_i)$. The distribution of $x_i$ in this case is given by
\begin{align*}
\Pr(x_i \geq t) 
&= \int_0^\infty \Pr(x_i \geq t \mid \lambda_i = \lambda) e^{-\lambda} d\lambda\\
&= \int_0^\infty e^{-(1+t)\lambda} d\lambda
= \frac{1}{1+t}\;,    
\end{align*}

which corresponds to $\phi(t) = 1+t$. More generally, we consider $\phi(t) = (1+t)^r$ for $r \in (\tfrac{1}{2},1]$. 
Such functions $\phi$ correspond to distributions with a heavy tale and with infinite expectation (i.e. $\int_0^\infty \frac{dt}{\phi(t)} = \infty$). Using Theorem \ref{thm:lower-bound}, they only yield lower bounds on the competitive ratio but not on $\ratio_{n,B}$. Corollary \ref{cor:lb-phir} shows the bound obtained for $r \to \frac{1}{2}$.

\begin{corollary}\label{cor:lb-phir}
Let $w \in [0,1]$. If $B = wn+o(n)$, then it holds for any algorithm $\A$ that
\[
\CR_B(\A) \geq (2 - w) - (3-2\sqrt{2}) w(1-w)\;.
\]
\end{corollary}

\begin{remark}
If $x_i$ is sampled from the distribution induced by $\phi(t) = (1+t)^r$, then $x_i+1$ follows a Pareto distribution with scale 1 and shape $r$ \cite{arnold2014pareto}, which is commonly used to model the distribution of job sizes in the context of the scheduling problem.
\end{remark}

\section{Known Partial Order}\label{sec:orderonly}

Before investigating the problem within the learning-augmented framework, we introduce an algorithm exclusively for the scenario with perfect information. Subsequently, in Section \ref{sec:switch}, we present a second algorithm, that we adapt to handle possibly erroneous predictions.

The optimal algorithm $\OPT$ does not necessitate precise knowledge of job sizes. Instead, it relies solely on their ordering. This observation suggests that it might be possible to improve the competitive ratio of the non-clairvoyant case by only knowing the relative order of a subset of the job sizes.
Therefore, rather than having access to the values $x_{\sigma(1)}, \ldots, x_{\sigma(B)}$, we assume that the decision-maker is only given a priority ordering $\pi$ of them, i.e. a bijection $\pi : [B] \to \sigma([B])$ satisfying $x_{\pi(1)} \leq \ldots \leq x_{\pi(B)}$. 

\begin{algorithm}[h!]
   \caption{Catch-up and Resume Round-Robin ($\CRRR$)}
   \label{algo:orderonly}
\begin{algorithmic}
   \STATE {\bfseries Input:} Ordering $\pi$ of $x_{\sigma(1)}, \ldots, x_{\sigma(1)}$
   \STATE Set $x_{\pi(0)} = 0$
   \FOR{$i=1$ {\bfseries to} $B$}
        \STATE Run job $\pi(i)$ for $x_{\pi(i-1)}$ units of time
        \WHILE{job $\pi(i)$ is not finished}
            \STATE Run round-robin on $\{\sigma(j)\}_{j>B} \cup \{\pi(i)\}$
        \ENDWHILE
   \ENDFOR
   \STATE Run round-robin on $\{\sigma(j)\}_{j=B+1}^n$ until completion
\end{algorithmic}
\end{algorithm}

In Algorithm \ref{algo:orderonly} ($\CRRR$), for all $i \in [B]$, the execution of job $\pi(i)$ starts only upon the completion of job $\pi(i-1)$. At this moment, all jobs $\sigma(j)$ for $j > B$ are either completed or have undergone execution for $x_{\pi(i-1)}$ units of time. $\CRRR$ then runs job $\pi(i)$ for a period of length $x_{\pi(i-1)}$ to catch up with the progress of the jobs $\{\sigma(j)\}_{j>B}$. Following this synchronization phase, it runs round-robin on the set of jobs $\{\sigma(j)\}_{j>B} \cup \{\pi(i)\}$ until $\pi(i)$ terminates. The same process iterates with $\pi(i+1)$ afterward.
Once all the jobs $\sigma(i)$ for $i \in [B]$ are completed, the algorithm runs round robin on the unfinished jobs in $\{\sigma(j)\}_{j >B}$.

Leveraging the ordering $\pi$, the algorithm aims to minimize the delays caused by longer jobs to shorter ones. In the ideal scenario where $B=n$, each job begins execution only after all shorter ones have been completed. When $B<n$, it is evident that the jobs $\{\sigma(i)\}_{i \in [B]}$ should be executed in the order specified by $\pi$. However, $\CRRR$ takes advantage of this ordering even further, ensuring that job $x_{\pi(i)}$ not only avoids delaying $x_{\pi(i-1)}$ but also does not delay any job $\sigma(j)$ with $j>B$ that has a size at most $x_{\pi(i-1)}$.


\begin{theorem}\label{thm:orderonly}
Algorithm  $\CRRR$ satisfies
\[
2 - \frac{B}{n}  
 - \frac{2(1-\frac{B}{n})}{(n+1)(B+1)}
\leq \ratio_{n,B}(\CRRR) 
\leq 2 - \frac{B}{n}\;.
\]
Moreover, if $B = \lfloor wn \rfloor$ for some $w \in [0,1]$, then $\CR(\CRRR) = 2 - w$.
\end{theorem}

Theorem \ref{thm:orderonly} shows a substantially stronger result than the one presented in \cite{benomar2023advice}, where the algorithm leveraging the values of the job sizes $x_{\sigma(1)}, \ldots, x_{\sigma(B)}$ is only $\big(2 - \frac{B(B-1)}{n(n-1)}\big)$-competitive.

\paragraph{Action predictions}
The information provided to $\CRRR$ is the order in which the jobs $\{\sigma(i)\}_{i \in [B]}$ would be executed by $\OPT$. This corresponds to the error-free scenario of \textit{action predictions} \cite{antoniadis2023online, lindermayr2022permutation, lassota2023minimalistic, sadek2024algorithms}, where the decision-maker receives predictions regarding the actions taken by the optimal offline algorithm, rather than numeric predictions of unknown parameters.
In the context of the scheduling problem, utilizing the $\ell_1$ norm to measure the error is not ideal for analyzing the action prediction setting \cite{im2021non}. Alternative error metrics, which account for the number of inversions in the predicted permutation in comparison to the true one \cite{lindermayr2022permutation}, would be more suitable. 
Therefore, adapting $\CRRR$ to imperfect action predictions is left for future research as it requires different considerations. For now, we shift our focus to introducing another algorithm that utilizes not only the priority order induced by the job sizes, but the size values themselves.

\section{Predictions of the Job Sizes}\label{sec:switch}

We propose in this section a generic algorithm $\switch$, which we will adapt in the cases of perfect and imperfect predictions. The algorithm takes as input $n$ jobs with unknown sizes and breakpoints $z_{\sigma(1)},\ldots,z_{\sigma(B)}$ that depend on the predictions of $x_{\sigma(1)},\ldots,x_{\sigma(B)}$, then it alternates running round-robin on the jobs $\{\sigma(j)\}_{j>B}$ and \textit{Shortest Predicted Job First} (SPJF), introduced in  \cite{purohit2018improving}, on the jobs $\{\sigma(i)\}_{i \in [B]}$, where the moment of switching from an algorithm to another is determined by the breakpoints.

As in Section \ref{sec:orderonly}, we call \textit{ordering} of $z_{\sigma(1)},\ldots,z_{\sigma(B)}$ any bijective application $\pi : [B] \to \sigma([B])$ satisfying $z_{\pi(1)} \leq \ldots \leq z_{\pi(B)}$. Note that, if the breakpoints are not pairwise distinct, then the ordering is not unique. In that case, $\switch$ chooses an ordering $\pi$ uniformly at random. We assume furthermore that the breakpoints induce the same order as the predictions, i.e. $z_{\sigma(i)} < z_{\sigma(j)} \iff y_{\sigma(i)} < y_{\sigma(j)}$ for all $i, j \in [B]$.

\begin{algorithm}[h!]
   \caption{Switch algorithm $\switch(z^\sigma,x)$}
   \label{algo:switch}
\begin{algorithmic}
   \STATE {\bfseries Input:} Breakpoints $z^\sigma = (z_{\sigma(i)})_{i \in [B]}$
   \STATE $\pi \gets$ ordering of $z^\sigma$ chosen uniformly at random\;
   \FOR{$i=1$ {\bfseries to} $B$}
        \WHILE{$\min\limits_{j>B} \frac{\spt_{\sigma(j)}(t)}{x_{\sigma(j)}} < 1$ and $\max\limits_{j > B}\spt_{\sigma(j)}(t) < z_{\pi(i)}$}
            \STATE Run round-robin on $\{\sigma(j)\}_{j=B+1}^n$
        \ENDWHILE
        \STATE Run job $\pi(i)$ until completion
   \ENDFOR
   \STATE Run round-robin on $\{\sigma(j)\}_{j=B+1}^n$ until completion
\end{algorithmic}
\end{algorithm}

Consider a run of $\switch$, and let $i \in [B]$. The first condition for entering the while loop is the existence of $j > B$ such that $\spt_{\sigma(j)}(t) < x_{\sigma(j)}$. This signifies that the jobs $\{\sigma(j)\}_{j>B}$ are not all completed, which is a verification feasible for the decision-maker without knowledge of the sizes $\{x_{\sigma(j)}\}_{j>B}$.
The second condition means that no job $\sigma(j)$ for $j>B$ has been in execution for more than $z_{\pi(i)}$ units of time. Given that round-robin allocates equal importance to all jobs $\{\sigma(j)\}_{j>B}$, upon exiting the while loop, each job $\sigma(j)$ is either completed or has been in execution for precisely $z_{\pi(i)}$ units of time.
Following this step, job $\pi(i)$ is executed until completion, and the same process recurs for $i+1$.

This algorithm ensures that any job $x_{\sigma(i)}$ with $i \leq B$ does not delay any other job $j$ whose size is at most $x_j \leq z_{\sigma(i)}$, and the delay it causes to jobs not satisfying this condition is exactly $x_{\sigma(i)}$. This allows efficient control of the mutual delays between the jobs by conveniently choosing the breakpoints.

\subsection{Perfect Predictions}
Assuming that the predictions are perfect, i.e. the decision-maker knows the exact sizes of jobs ${\sigma(1)}, \ldots, {\sigma(B)}$, it is possible to set $z_{\sigma(i)} = x_{\sigma(i)}$ for all $i \in [B]$.

\begin{theorem}\label{thm:switch-perfect}
Algorithm $\switch$ with breakpoints $z_{\sigma(i)} = x_{\sigma(i)}$ for all $i \in [B]$ satisfies
\[
\ratio_{n,B}(\switch) 
= 2 - \frac{B}{n} - \frac{2(1 - \frac{B}{n})}{n+1}\;.
\]
In particular, if $B = \lfloor wn \rfloor$ for some $w \in [0,1]$ then $\CR_B(\switch) = 2 - w$.
\end{theorem}

Note that the $(n,B)$-competitive ratio above is strictly better than that of $\CRRR$, presented in Theorem \ref{thm:orderonly}. However, both algorithms have equivalent performance when $n$ is large. In particular, their competitive ratios coincide when $B = \lfloor wn \rfloor$.

A slight improvement on the $(n,B)$-competitive ratio can be obtained by introducing randomness into $\switch$. Indeed, consider the \textit{Run To Completion} algorithm ($\textsf{RTC}$) defined in \cite{motwani1994nonclairvoyant}, executing all the jobs until completion in a uniformly random order. Then we have the following result.

\begin{proposition}\label{prop:rand-switch}
The algorithm that runs $\textsf{RTC}$ with probability $\frac{2(n-B)}{n(n+3) - 2B}$, and runs $\switch$ with breakpoints $z_{\sigma(i)} = x_{\sigma(i)}$ for all $i \in [B]$ with the remaining probability, has an $(n,B)$-competitive ratio of
\[
2 - \frac{B}{n} - \frac{2(1 - \frac{B}{n})(2 - \frac{B}{n})}{n+3 - \frac{2B}{n}}\;.
\]
\end{proposition}

For $B=0$, the ratio above becomes $2 - \frac{4}{n+3}$, which is the best possible in the non-clairvoyant setting.


\subsection{Imperfect Predictions}\label{sec:imperfect-pred}
We assume in this section that no quality guarantees are given on the predictions $\{y_{\sigma(i)}\}_{i \in [B]}$. Recall that the total error $\eta^\sigma = \sum_{i=1}^B |x_{\sigma(i)} - y_{\sigma(i)}|$ is a random variable because $\sigma$ is a uniformly random permutation of $[n]$, hence our results will depend on $\E[\eta^\sigma]$.

The goal
is to design an algorithm that is consistent, robust, and with a competitive ratio having a smooth dependency to $\E[\eta^\sigma]$.
We first study the consistency and smoothness of $\switch$ with well-chosen breakpoints, then we show that combining it with round-robin as in \cite{purohit2018improving, lassota2023minimalistic} gives robustness guarantees.

Using the trivial breakpoints $z_{\sigma(i)} = y_{\sigma(i)}$ as in the previous section is not enough to guarantee smoothness. Consider, for example, job sizes all equal to 1, and $B$ predictions $y_{\sigma(i)} = 1 - \epsilon$ for an arbitrarily small $\epsilon$. Blindly following these predictions, taking $z_{\sigma(i)} = y_{\sigma(i)}$ for all $i \in [B]$, results in delaying all jobs with unknown sizes by $B$ time units compared to the case of perfect predictions. This creates a discontinuity in the competitive ratio when $\epsilon$ becomes positive, proving non-smoothness. Hence, we consider instead randomized breakpoints. 

\begin{algorithm}[h!]
   \caption{imperfect predictions $\switch(\xi y^\sigma, x)$}
   \label{algo:any-pred}
\begin{algorithmic}
   \STATE {\bfseries Input:} predictions $(y_{\sigma(i)})_{i \in [B]}$, distribution $F$
   \STATE Sample $\xi \sim F$\;
   \STATE Run $\switch$ with breakpoints $z_{\sigma(i)} = \xi y_{\sigma(i)}$
\end{algorithmic}
\end{algorithm}

Algorithm \ref{algo:any-pred} simply runs $\switch$ with random breakpoints $z_{\sigma(i)} = \xi y_{\sigma(i)}$. 
The following lemma gives an upper bound on the algorithm's objective function depending on the distribution $F$ of $\xi$.

\begin{lemma}\label{lem:alg-output-F}
Let $F$ be a probability distribution on $(0,\infty)$, and consider the mappings $h_F: (0,\infty)^2 \to \R$ and $g_F: (0,\infty) \to \R$ defined by
\begin{align*}
h_F(s,t) &= t \cdot {\textstyle \Pr_{\xi\sim F}}(\xi < \tfrac{s}{t})\\
g_F(s) &= (1-s) {\textstyle \Pr_{\xi\sim F}}(\xi<s) + \E_{\xi\sim F}[\xi \indic{\xi < s}]\;.  
\end{align*}
Let $\beta_F = \sup_{s \in (0, 1]}\frac{g_F(s)}{s}$ and $\gamma_F = \sup_{s \geq 1} (g_F(s) + s)$. If $\beta_F, \gamma_F < \infty$ and $h_F$ is $L_F$-Lipschitz w.r.t. the second variable $t$, then for any job sizes $x_1,\ldots,x_n$ and $B \leq n$, the expected sum of the completion times achieved by $\switch$ with breakpoints $z_{\sigma(i)} = \xi y_{\sigma(i)}$  is at most 
\begin{align*}
\sum_{i=1}^n x_i + C^1_{n,B,F} \sum_{i<j} \min(x_i,x_j) + C^2_{n,B,F} \E[\eta^\sigma]\;,    
\end{align*}
with $C^1_{n,B,F} = 2 - \tfrac{B}{n} - \big(2 - \beta_F - \gamma_F \big)\tfrac{B}{n}\big( 1 - \tfrac{B-1}{n-1}\big)$ and $C^2_{n,B,F} = (1+L_F+\E[\xi])(n-B) + B-1$.
\end{lemma}

A trivial choice of $\xi$ is the constant random variable equal to $1$ a.s., but this is not enough to guarantee smoothness, as it corresponds to the Dirac distribution $F = \delta_1$, for which $h_F$ is not continuous w.r.t. to $t$. 
In the next lemma, we provide a specific choice of distribution $F$ depending on a single parameter $\rho$, and we express the upper bound from the previous lemma using this parameter.

\begin{lemma}\label{lem:switch-exp}
Let $\rho \in (0,1]$ and 
\[
F: s \mapsto (1 - e^{-(s-1)/\rho})\indic{s>1}
\]
a shifted exponential distribution with parameter $1/\rho$, i.e. $\xi \sim 1 + \mathcal{E}(1/\rho)$, then $\switch$ with breakpoints $z_{\sigma(i)} = \xi y_{\sigma(i)}$ for all $i \in [B]$ has an $(n,B)$-competitive ratio of at most
\[
\left( 2 - \tfrac{B}{n} + \rho \tfrac{B}{n}(1 - \tfrac{B-1}{n-1} ) \right) + \left( \tfrac{4}{\rho}(1 - \tfrac{B}{n}) + \tfrac{B}{n} \right) \frac{n \E[\eta^\sigma]}{\OPT(x)}\;.
\]
\end{lemma}

The previous lemma highlights a tradeoff between the smoothness and the consistency of the algorithm. Indeed, as $\rho$ decreases, the algorithm gains in consistency, but the term $\left( \tfrac{4}{\rho}(1 - \tfrac{B}{n}) + \tfrac{B}{n} \right)$ multiplying $\E[\eta^\sigma]$ becomes larger. However, while setting $\rho$ close to zero results in an arbitrarily high sensitivity to the error, setting it close to $1$ gives consistency of at most $\big(2 - \frac{B(B-1)}{n(n-1)}\big)$, which is still a decreasing function of $B$, interpolating the values $2$ and $1$ in the clairvoyant and non-clairvoyant cases. This implies that sacrificing a small amount of consistency significantly improves smoothness.

For $B = n$, assuming that all the job sizes are at least $1$, it holds that $\OPT(x) \geq n(n+1)/2$ and the lemma gives $\ratio_{n,B}(\eta;\switch) \leq 1 + \frac{2 \eta}{n}$, matching the bound proved on (SPJF) in \cite{purohit2018improving}. On the other hand, if $\eta = 0$, setting $\rho = 0$ results in $\ratio_{n,B}(0;\switch) \leq 2 - \frac{B}{n}$. Using tighter inequalities in the proof, it is possible to retrieve the bound established in Theorem \ref{thm:switch-perfect} (See Inequality \eqref{ineq:precise-bound-ratio-pred} in Appendix \ref{appx:switch}).

\paragraph{Preferential algorithm}
Now we need to adapt the algorithm to guarantee robustness in the face of arbitrarily erroneous predictions. We use the same approach as Lemma 3.1 of \cite{purohit2018improving}, which consists of running concurrently a consistent algorithm and round-robin at respective rates $\lambda$, $1-\lambda$ for some $\lambda \in [0,1]$. However, their result only applies for deterministic algorithms $\A$ satisfying for any instances $x = (x_1,\ldots,x_n)$ and $x' = (x'_1,\ldots,x'_n)$ that
\[
\big(\forall i \in [n]: x_i \leq x'_i\big)
\implies \A(x) \leq \A(x')\;.
\]
Such algorithms are called \textit{monotonic}. $\switch$ with breakpoints $z_{\sigma(i)} = \xi y_{\sigma(i)}$ is not deterministic since its objective function depends both on $\sigma$ and $\xi$. Nonetheless, we overcome this difficulty by proving that, conditionally to $\sigma$ and $\xi$, its outcome is deterministic and monotonic, then we establish the following theorem.

\begin{theorem}\label{thm:preferential-algo}
Let $\rho \in (0,1]$ and $F = 1 + \mathcal{E}(1/\rho)$. Then the preferential algorithm $\ALG_\lambda$ which runs Algorithm \ref{algo:any-pred} at rate $\lambda$ and round-robin at rate $1-\lambda$ has an $(n,B)$-competitive ratio of at most
\[
\min\left( \frac{2}{1 - \lambda} \;,\; \frac{C_{\rho,n,B}}{\lambda} +  \frac{S_{\rho,n,B}}{\lambda} \cdot \frac{n \E[\eta^\sigma] }{\OPT(x)}\right)\;,
\]
with
\begin{align*}
C_{\rho,n,B} &= \big(2 - \tfrac{B}{n}\big) + \rho \tfrac{B}{n}\big(1 - \tfrac{B-1}{n-1}\big) \\
S_{\rho,n,B} &= \tfrac{4}{\rho}\big(1 - \tfrac{B}{n}\big) + \tfrac{B}{n} \;.
\end{align*}
\end{theorem}

This upper bound generalizes that of \cite{purohit2018improving}. It presents a consistency-robustness tradeoff that can be tuned by adjusting the parameter $\lambda$, and a consistency-smoothness tradeoff controlled by the parameter $\rho$, which vanishes for $B$ close to $0$ or $n$, as the terms multiplying $\rho$ and $1/\rho$ respectively in $C_{\rho,n,B}$ and $S_{\rho,n,B}$ become zero.

\section{Experiments}

In this section, we validate our theoretical findings by testing the algorithms we presented on various benchmark job sizes. In all the figures, each point is averaged over $10^4$ independent trials.

\paragraph{Perfect information}
We test the performance of Algorithms $\switch$ and $\CRRR$ against the hard instances used to prove the lower bounds of Section \ref{sec:lowerbounds}: 
we consider i.i.d. job sizes sampled from the exponential distribution with parameter $1$, and job sizes drawn from the distribution $\Phi(r,a)$ with parameters $r=0.51$ and $a=10^4$, characterized by the tail probability
\[
\Pr(x^a_i \geq t)
= \frac{(1+t)^{-r} - (1+a)^{-r}}{1 - (1+a)^{-r}} \indic{t < a}\;,
\]
This distribution is a truncated version of the one defined by $\Pr(x^\infty_i \geq t) = \frac{1}{(1+t)^r}$. The bound of Corollary \ref{cor:lb-phir} is obtained by using this distribution for $a>0$ and $r \in (\frac{1}{2},1)$, and taking the limits $n \to \infty$, $a \to \infty$, and $r\to 1/2$.

\begin{figure}[h!]
\begin{center}
\centerline{\includegraphics[width=\columnwidth]{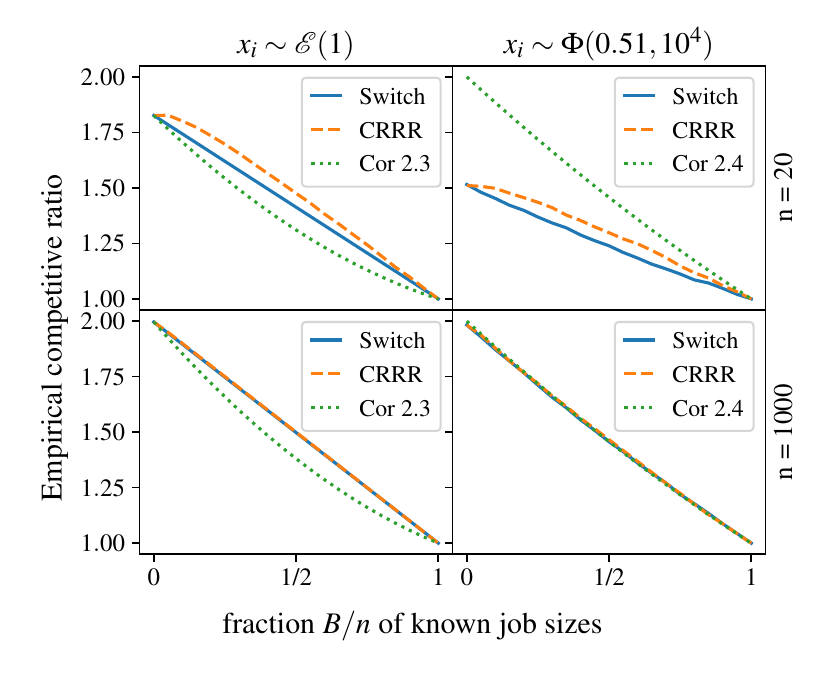}}
\caption{Lower bounds and ratios of $\switch$, $\CRRR$}
\label{fig:perfectPred}
\end{center}
\vskip -0.3in
\end{figure}

Figure \ref{fig:perfectPred} exhibits the empirical ratios achieved by both algorithms with a number $n \in \{20,1000\}$ of jobs. For $n = 20$, $\switch$ outperforms $\CRRR$ for the two distributions, whereas their ratios are very close for $n=1000$. This confirms that $\switch$ and $\CRRR$ are asymptotically equivalent, as can be deduced from Theorems \ref{thm:orderonly} and \ref{thm:switch-perfect}.
For the exponential distribution, as expected, both algorithms have ratios above the non-asymptotic lower bound of Corollary \ref{cor:lb-exp}.
Meanwhile, considering the distribution $\Phi(0.51,10^4)$, the empirical ratios for $n = 20$ are below the lower bound of Corollary \ref{cor:lb-phir}, because it is proved by taking $n \to \infty$. For $n=1000$, the ratios match the lower bound.

\paragraph{Preferential algorithm} 
In the remaining discussion, we refer to $\switch$ with breakpoints $z_{\sigma(i)} = \xi y_{\sigma(i)}$ and $\xi \sim 1 + \mathcal{E}(1/\rho)$, as $\switch$ with parameter $\rho$.

We generate a synthetic instance of $n = 50$ job sizes, drawn independently from the Pareto distribution with scale 1 and shape 1.1. The Pareto distribution, known for its heavy tail, is particularly suitable for modeling job sizes \cite{harchol1997exploiting, bansal2001analysis, arnold2014pareto}, and it is a commonly used benchmark for learning-augmented scheduling algorithms \cite{purohit2018improving, lindermayr2022permutation}. Furthermore, we consider noisy predictions $y_i = x_i + \eps_i$ for all $i \in [50]$, where $\eps_i$ is sampled independently from a normal distribution with mean $0$ and standard deviation $\tau$.

\begin{figure}[h!]
\begin{center}
\centerline{\includegraphics[width=\columnwidth]{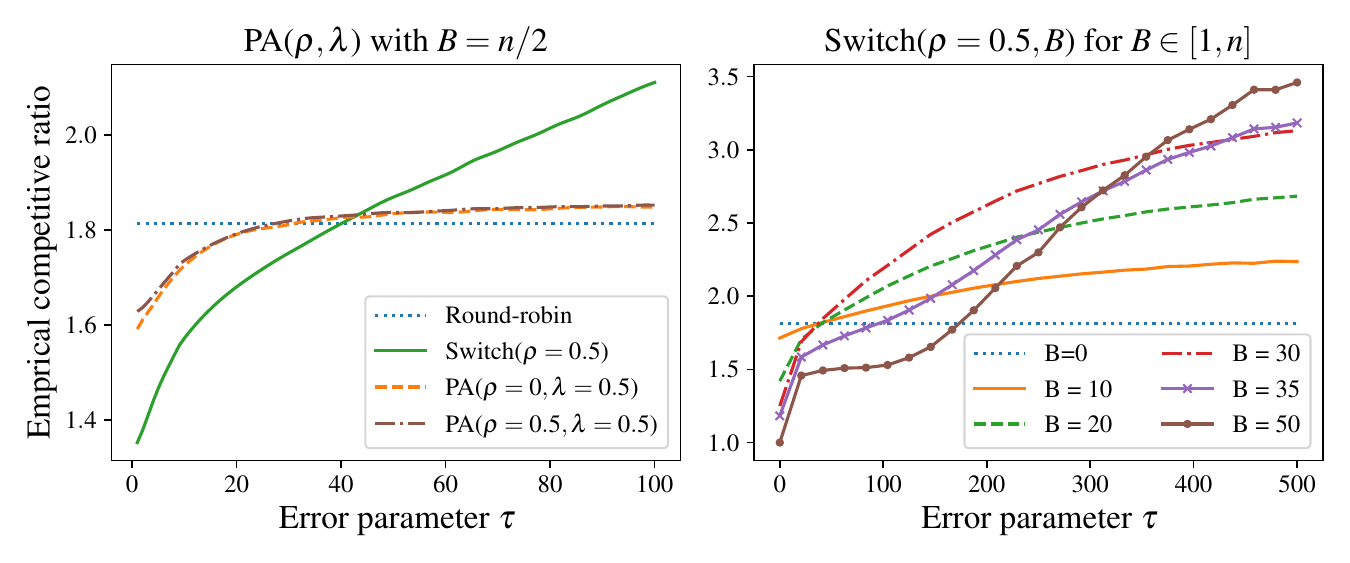}}
\caption{Preferential Algorithm (PA) with different parameters}
\label{fig:PAparameters}
\end{center}
\vskip -0.3in
\end{figure}

Figure \ref{fig:PAparameters} illustrates the empirical ratio of the Preferential Algorithm (PA) across various parameter configurations, with varying error parameter $\tau$.

The left plot displays the ratios for different $\lambda$ and $\rho$ values, with $B = 25 = n/2$. When $\lambda = 0$, PA becomes round-robin. For $\lambda=1$ and $\rho=0.5$, PA simply runs $\switch(\rho=0.5)$, which gives an improved consistency ($\tau = 0$), not equal to $1$ because $B<n$ and $\rho > 0$, and gives a ratio that deteriorates arbitrarily as $\tau$ increases. 
In contrast, PA with $\lambda = 0.5$ gives a weaker consistency but maintains bounded ratios, even with arbitrarily erroneous predictions. The choice of $\rho = 0$ exhibits a slightly better consistency compared to $\rho=0.5$, in line with theoretical expectations, but there is no significant difference regarding sensitivity to errors. This should not be surprising since setting $\rho > 0$ ensures smoothness in the worst-case (see Figure \ref{fig:consistency-smoothness}), but it is not necessarily needed for all instances.

The right plot examines the influence of $B$ on PA with parameters $\lambda = 1$ and $\rho=0.5$, which corresponds to
$\switch$ with $\rho = 0.5$. Larger $B$ values improve consistency and also yield a smaller sensitivity to small errors. However, for high $\tau$ values, having numerous predictions leads to faster performance deterioration compared to having fewer predictions. This shows that more predictions enhance consistency, while fewer predictions enhance robustness.

\paragraph{Consistency-smoothness}
To shed light on the tradeoff between consistency and smoothness raised in Section \ref{sec:imperfect-pred}, we consider i.i.d. job sizes $x_1,\ldots,x_{100}$, each taking the value 1 w.p. 1/2 and 2 w.p. 1/2, and we consider noisy predictions of the form $y_i = x_i + \eps_i$, where $\eps_i$ follows a uniform distribution over $[-\tau, \tau]$. Figure \ref{fig:consistency-smoothness} illustrates the evolution, for $\tau$ varying in $[0,0.15]$, of the empirical competitive ratio of $\switch$ with parameter $\rho \in \{0, 0.1, 0.5\}$ and $B \in \{50,95\}$,

For both values of $B$, the experiment reveals that larger values of $\rho$ give bigger ratios when $\tau = 0$ (less consistency), but on the other hand they yield less sensitivity to variations of the expected prediction error (better smoothness), which confirms our theoretical findings. In particular, for $\rho = 0$, a significant discontinuity arises when $\tau$ becomes positive. Figure \ref{fig:consistency-smoothness} also shows that this tradeoff is less significant as $B$ approaches $n = 100$,  with the consistency values for $\rho \in \{0, 0.1, 0.5\}$ drawing closer.


\begin{figure}[h!]
\begin{center}
\centerline{\includegraphics[width=\columnwidth]{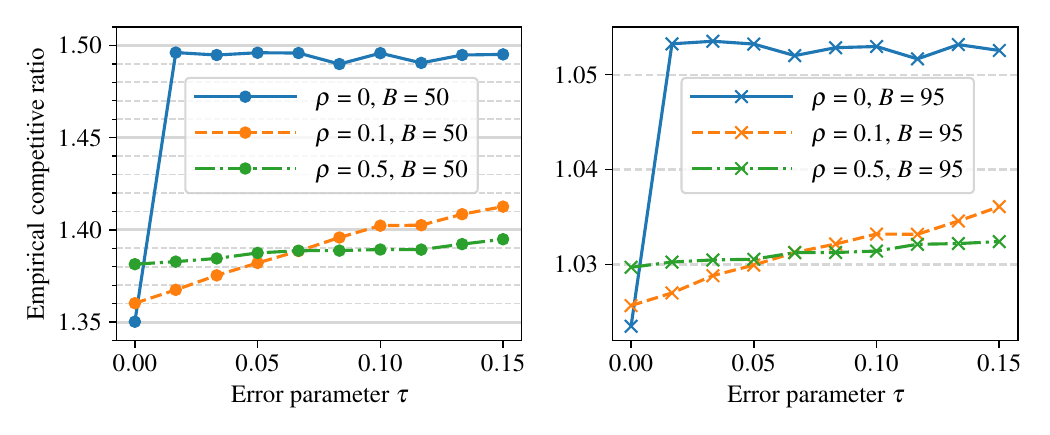}}
\caption{Tradeoff between consistency and smoothness}
\label{fig:consistency-smoothness}
\end{center}
\vskip -0.3in
\end{figure}

\section{Conclusion and Future Work}

This paper explores the non-clairvoyant scheduling problem with a limited number of predicted job sizes. We give near optimal lower and upper bounds in the case of perfect predictions, and we introduce a learning-augmented algorithm raising the common consistency-robustness tradeoff and an additional consistency-smoothness tradeoff, the latter vanishing when $B$ approaches $0$ or $n$.

Our findings join previous works in demonstrating that online algorithms can indeed achieve improved performance even when armed with a restricted set of predictions, which is an assumption more aligned with practical scenarios. Furthermore, they affirm the necessity of studying and understanding these regimes, as they may unveil unique behaviors absent in the zero- or full-information settings.

\subsection{Open Questions}

\paragraph{Tight lower bounds} In the case of perfect predictions, there is
a (small) gap between the lower bounds of Section \ref{sec:lowerbounds} and the competitive ratios of $\switch$ and $\CRRR$. An interesting research avenue is to close this gap, either by designing better algorithms or improving the lower bound. This could involve using Theorem \ref{thm:lower-bound} with more refined distributions.

\paragraph{Reduced number of action predictions} Algorithm $\switch$ leverages the job sizes' predictions, not only the order they induce. Using the $\ell_1$ norm to measure the error is thus a suitable choice. However, as discussed in Section \ref{sec:orderonly}, Algorithm $\CRRR$ only uses the priority order in which $\OPT$ runs $(x_{\sigma(i)})_{i \in [B]}$. An interesting question to explore is how to adapt it in the case of imperfect \textit{action predictions}, using appropriate error measures.

\paragraph{Smooth and $(2-\frac{B}{n})$-consistent algorithm}
Lemma \ref{lem:switch-exp} and Figure \ref{fig:consistency-smoothness} emphasize that, to achieve smoothness, $\switch$ with parameter $\rho$ must exhibit a consistency exceeding $2 - \frac{B}{n}$. A compelling question arises: Is it possible to devise a smooth algorithm with a consistency of at most $2 - \frac{B}{n}$? 

\section*{Impact Statement}
This paper presents a work whose goal is to advance the field of learning-augmented algorithms. There are many potential societal consequences of our work, none of which we feel must be specifically highlighted here.

\section*{Acknowledgements}
Vianney Perchet acknowledges support from the French National Research Agency (ANR) under grant number (ANR19-CE23-0026 as well as the support grant, as well as from the grant “Investissements d’Avenir” (LabEx Ecodec/ANR11-LABX-0047).


\bibliography{bibliography}
\bibliographystyle{icml2024}

\appendix
\onecolumn

\section{Lower bounds}

\subsection{Preliminary result}
\begin{lemma}\label{lem:yao}
Let $F$ be a probability distribution on $(0,\infty)$ with finite expectation, $x$ an array of $n$ i.i.d. random variables sampled from $F$, and $\alpha_{n,B} \geq 0$ satisfying for any deterministic algorithm $\A$ with access the sizes of $B$ jobs that 
\[\frac{\E_{\sigma, x} [\A(x)]}{\E_{x}[\OPT(x)]} \geq \alpha_{n,B}\;,\]
then for any (deterministic or randomized) algorithm $\ALG$, we have 
\[\ratio_{n,B} (\ALG) \geq \alpha_{n,B}\;.\]
\end{lemma}

\begin{proof}
Let $F$ and $\alpha_{n,B}$ be as stated in the lemma. Using Yao's minimax principle \cite{yao1977probabilistic} we deduce that for any randomized algorithm $\ALG$ 
\begin{align*}
\E_{\sigma, x}[\ALG(x)]
&\geq \inf_{\A \text{ deterministic}} \E_{\sigma,x} [\A(x)]\\
&\geq \alpha_{n,B} \E_{x}[\OPT(x)]\;,    
\end{align*}
where the infimum is taken over all deterministic algorithms. The previous inequality can be written as
\[
\E_{x}\big[ \E_\sigma[\ALG(x)] - \alpha_{n,B} \OPT(x)\big] \geq 0\;,
\]
and this implies that, necessarily, there exists a value $x^* = (x^*_n, \ldots, x^*_n)$ taken by $x$ verifying $\E_\sigma[\ALG(x^*)] - \alpha_{n,B} \OPT(x^*) \geq 0$, hence
\[
\ratio_{n,B} (\ALG) 
\geq \frac{\E_\sigma[\ALG(x^*)]}{\OPT(x^*)}
\geq \alpha_{n,B}\;.
\]
\end{proof}

\subsection{Proof of Theorem \ref{thm:lower-bound}}

Let $\A$ be a deterministic algorithm. Using Equation \ref{eq:generic-output}, it suffices to bound $\E[P^\A_{ij}]$ for all $i \neq j$ to deduce a bound on $\E[\A(x)]$, and since $P^\A_{ij}$ is a non-negative random variable, the focus can be narrowed down to bounding $\Pr(P_{ij}^\A > t)$ for all $t > 0$. Following the proof scheme of \cite{motwani1994nonclairvoyant}, we introduce the following definition.

\begin{definition}\label{def:uij}
Let $\A$ be a deterministic algorithm given an instance of $n$ jobs. For all $i\neq j \in [n]$ and $t \geq 0$, we denote by $\uij_{i,j}^\A (t)$ the time spent on job $i$ when the total time spent on both jobs $i$ and $j$ is $t$, assuming neither job $i$ nor $j$ is completed. More precisely, $\uij_{i,j}^\A(t)$ is defined by
\[
\uij_{i,j}^\A(t) = \spt^\A_i\big(\inf\{t' \geq 0: \spt^\A_i(t') + \spt^\A_j(t') = t\} \big)\;.
\]
\end{definition}

$\uij_{i,j}^\A$ is therefore a rule defined solely by the algorithm. For an instance $x = (x_1,\ldots,x_n)$ of job sizes, the real time spent on $i$ when a total time of $t$ has been spent on both jobs $i,j$ is given by $\min(\uij_{i,j}^\A(t), x_i)$.

For all $t\geq 0$, the following Lemma allows to express the event $P^\A_{ij} > t$ using $u_{ij}^\A(t)$.

\begin{lemma}\label{lem:P-uij}
Let $x_1,\ldots,x_n$ be instance of $n$ job sizes, then for any algorithm $\A$, for any $i \neq j \in [n]$ and $t \geq 0$, the following equivalence holds
\[
\big(P^\A_{ij} > t \big)
\iff
\big(x_i > \uij^\A_{i,j}(t) \text{ and } x_j > t - \uij^\A_{j,i}(t)\big)
\]
\end{lemma}

\begin{proof}
Let $\A$ be an algorithm and $i \neq j \in [n]$. We denote by $\spt_{ij}(t) = \spt_i(t) + \spt_j(t)$ the total processing time spent on both jobs $i$ and $j$ up to time $t$. 
Assuming that job $i$ finishes first, i.e. $t_i \leq t_j$, no processing time is spent on job $i$ after $t_i$, hence $\spt_i(t_j) = \spt_i(t_i) = x_i$, and $P_{ij} = \spt_i(t_j) + \spt_j(t_i) = \spt_{ij}(t_i)$. By symmetry, we deduce that $P_{ij} = \spt_{ij}(\min(t_i,t_j))$. Therefore, using that $\spt_i$ and $\spt_j$ are non-decreasing and continuous, it holds for all $t \geq 0$
\begin{align}
P_{ij} > t
&\iff \spt_{ij}(\min(t_i,t_j)) > t \nonumber\\
&\iff \min(t_i,t_j) > \inf\{t' : \spt_{ij}(t') \geq t\} \nonumber\\
&\iff x_i > \spt_i(\inf\{t' : \spt_{ij}(t') \geq t\}) \text{ and } x_j > \spt_j(\inf\{t' : \spt_{ij}(t') \geq t\}) \label{eq:ti>s}\\
&\iff x_i > \uij_{ij}(t) \text{ and } x_j > t - \uij_{ij}(t)\label{eq:uij-uji}\;.
\end{align}
Equivalence \eqref{eq:ti>s} holds because $t_i = \inf\{t' \geq 0: S_i(t') \geq x_i\}$, thus for any $s \geq 0$ we have $t > s \iff x_i > \spt_i(s)$. The same holds for $j$. For Equivalence \eqref{eq:uij-uji}, we simply used Definition \ref{def:uij} and the observation $\uij_{ij}(t) + \uij_{ji}(t) = t$.
\end{proof}

\begin{remark}\label{rmk:uij-nonclairvoyant}
In the case of non-clairvoyant algorithms, the rule defining $u^\A_{ij}(\cdot)$ is dictated by the algorithm, independent of the job sizes. Thus, if the job sizes are sampled independently from the exponential distribution, Lemma \ref{lem:P-uij} gives for all $i \neq j$ that $\Pr(P^\A_{ij}>t) = \Pr(x_i > \uij^\A_{i,j}(t)) \Pr( x_j > t - \uij^\A_{j,i}(t)) = e^{-t}$, and it follows immediately that $\E[P^\A_{ij}] = 1$ for any deterministic algorithm. This argument, used in \cite{motwani1994nonclairvoyant}, is not applicable in our context since the algorithm possesses access to certain job sizes, enabling the formulation of a rule for $u^\A_{i,j}$ that considers this information. Therefore, more sophisticated techniques become necessary for our analysis since the independence of the events $x_i > \uij^\A_{i,j}(t)$ and $x_j > t - \uij^\A_{j,i}(t)$ is lost.    
\end{remark}

\begin{lemma}\label{lem:generic-lb}
Let $a \in (0,\infty)$ and $\phi : [0,\infty) \to [0,\infty)$ a continuously differentiable and increasing function satisfying that $\phi(0) > 0$, $t \mapsto \frac{\phi'(t)}{\phi(t)}$ is non-increasing and  $\int_0^\infty \frac{dt}{\phi(t)^2} < \infty$. Let $x_1^a,\ldots,x_n^a$ i.i.d. random job sizes with distribution 
\[
\Pr(x_1^a \leq t)
= 1 - \frac{\phi(t)^{-1} - \phi(a)^{-1}}{\phi(0)^{-1} - \phi(a)^{-1}} \indic{t < a}\;,
\]
then for any algorithm $\A$ having access to the sizes of the first $B$ jobs $x^a_1,\ldots,x^a_B$, it holds that
\begin{alignat*}{3}
\E[P^\A_{ij}] 
&\geq \int_0^\infty \frac{\phi(0)^2}{\phi(t)^2}dt - \underset{a \to \infty}{o(1)} 
&&\forall i\neq j \leq B\;,\\
\E[P^\A_{ij}] 
&\geq 2\int_0^\infty\frac{\phi(0)^2}{\phi(t)^2}dt - \underset{a \to \infty}{o(1)} 
&&\forall i\neq j > B\;,\\
\E[P^\A_{ij}] 
&\geq \phi(0)^2 \int_0^\infty \left\{\inf_{T \geq x} G_\phi(x,T)\right\} \frac{\phi'(x)}{\phi(x)}dx - \underset{a \to \infty}{o(1)} 
\qquad &&\forall i\leq B, j > B\;,
\end{alignat*}
where $G_\phi(x,T)$ is defined in Equation \eqref{eq:Gphi}, and the $\underset{a \to \infty}{o(1)}$ term does not depend on the algorithm $\A$.
\end{lemma}

Although Lemma \ref{lem:generic-lb} is stated with $a \in (0,\infty)$, the results also hold for random variables $x^\infty_1,\ldots,x^\infty_n$ sampled from the limit distribution $\Pr(x_1^\infty \leq t) = 1 - \frac{\phi(0)}{\phi(t)}$, where the $o(1)$ term becomes zero.

Before proving the lemma, let us first observe that, since $\phi$ is increasing and $\int_0^\infty \frac{dt}{\phi(t)^2} < \infty$, then necessarily $\lim_{x\to \infty}\phi(x) = \infty$.

\begin{proof}
Let us first observe that for all $t \geq 0$
\[
\Pr( \min(x^\infty_1,x^\infty_2) \geq t )
= \Pr(x^\infty_1 \geq t) \Pr(x^\infty_2 \geq t)
= \frac{\phi(0)^2}{\phi(t)^2}\;,
\]
hence
\[
\E[\min(x^\infty_1,x^\infty_2)] 
= \int_0^\infty \Pr( \min(x^\infty_1,x^\infty_2) \geq t ) dt
= \int_0^\infty \frac{\phi(0)^2}{\phi(t)^2}dt \;.
\]
In the following, we prove separately the three claims of the Lemma: in the cases where both jobs sizes $x_i, x_j$ are known, both are unknown, and where only $x_i$ is known.

\paragraph{Both job sizes are known}
For any $i \neq j \in [n]$, it holds that $P_{ij} \geq \min(x^a_i,x^a_j)$, and this true in particular for $i \neq j \in [B]$.
Furthermore, for any $t \geq 0$, since $\phi(t) \geq \phi(0)$, we have that the mapping $a \mapsto \frac{\phi(t)^{-1} - \phi(a)^{-1}}{\phi(0)^{-1} - \phi(a)^{-1}}$ is non-increasing. It follows that
\begin{align}
\Pr(x_1^a \leq t) 
&\leq 1 - \left(\lim_{a' \to \infty} \frac{\phi(t)^{-1} - \phi(a')^{-1}}{\phi(0)^{-1} - \phi(a')^{-1}}\right) \indic{t < a} \nonumber\\
&= 1 - \frac{\phi(0)}{\phi(t)} \indic{t < a}\;, \label{eq:Pxa>Px}    
\end{align}
and therefore, for any $i \neq j \leq B$
\begin{align*}
\E[P_{ij}]
&\geq \E[\min(x_1^a, x_2^a)] \\
&= \int_{0}^\infty \Pr(x_1^a \geq t)^2 dt\\
&\geq \int_{0}^\infty \frac{\phi(0)^2}{\phi(t)^2} \indic{t < a} dt
= \int_{0}^a \frac{\phi(0)^2}{\phi(t)^2} dt   \\
&= \int_0^\infty\frac{\phi(0)^2}{\phi(t)^2}dt - \underset{a \to \infty}{o(1)}\;.
\end{align*}

\paragraph{Both job sizes are unknown}
For $i \neq j > B$, the algorithm ignores the job sizes $x^a_i$ and $x^a_j$ 
Therefore, $\uij_{i,j}$ is independent of $x^a_i$ and $x^a_j$. Consequently, using Lemma \ref{lem:P-uij}, the independence of $x_i^a$ and $x_j^a$, then Inequality \eqref{eq:Pxa>Px}, we obtain
\begin{align*}
\Pr(P_{ij} > t) 
&= \Pr(x_i > \uij_{i,j}(t) \text{ and } x_j > t - \uij_{i,j}(t))\\
&= \Pr(x_i > \uij_{i,j}(t)) \Pr(x_j > t - \uij_{i,j}(t))\\
&\geq \left(\frac{\phi(0)}{\phi(\uij_{i,j}(t))} \indic{\uij_{i,j}(t) < a} \right) \left( \frac{\phi(0)}{\phi(t - \uij_{i,j}(t))} \indic{t -\uij_{i,j}(t) < a}  \right)\\
&= \frac{\phi(0)^2 \indic{t - a< \uij_{i,j}(t) < a}}{\phi(\uij_{i,j}(t) \phi(t - \uij_{i,j}(t))} \;.
\end{align*}
Using that $\frac{\phi'}{\phi}$ is non-increasing, we have for any $t,u \geq 0$ that 
\begin{align*}
\frac{d}{du}\left( \phi(u) \phi(t-u) \right) 
&= \phi'(u)\phi(t-u) - \phi(u) \phi'(t-u)\\
&= \phi(u)\phi(t-u)\left( \frac{\phi'(u)}{\phi(u)} - \frac{\phi'(t-u)}{\phi(t-u)} \right)
\end{align*}
The sign of $\frac{d}{du}\left( \phi(u) \phi(t-u) \right)$ is the same as that of $\frac{\phi'(u)}{\phi(u)} - \frac{\phi'(t-u)}{\phi(t-u)}$, which is a non-increasing function. It is null for $u = t/2$, hence it is non-positive for $u \leq t/2$ and non-negative for $u \geq t/2$. This implies that $u \mapsto \phi(u) \phi(t-u)$ is minimal for $u = t/2$, i.e. $\phi(u)\phi(t-u) \leq \phi(\tfrac{t}{2})^2$ for all $t,u \geq 0$, in particular
\[
\Pr(P_{ij} > t) \geq \frac{\phi(0)^2}{\phi(t/2)^2} \indic{t - a< \uij_{i,j}(t) < a}\;,
\]
and it follows that
\begin{align*}
\E[P_{ij}] 
&= \int_0^\infty \Pr(P_{ij} > t) dt
\geq \int_0^\infty \frac{\phi(0)^2}{\phi(\tfrac{t}{2})^2} \indic{t - a< \uij_{i,j}(t) < a} dt\;.
\end{align*}
Observing that, for any $t \geq 0$, the mapping $a \mapsto \phi(0)^2/\phi(\tfrac{t}{2})^2 \indic{t - a< \uij_{i,j}(t) < a}$ is non-decreasing, and its limit is $\phi(0)^2/\phi(\tfrac{t}{2})^2$, the monotone convergence theorem guarantees that 
\begin{align*}
\lim\limits_{a \to \infty} \int_0^\infty \frac{\phi(0)^2}{\phi(\tfrac{t}{2})^2} \indic{t - a< \uij_{i,j}(t) < a} dt
&= \int_0^\infty \frac{\phi(0)^2}{\phi(\tfrac{t}{2})^2} dt\\
&= 2 \int_0^\infty \frac{\phi(0)^2}{\phi(t)^2} dt\;,    
\end{align*}
therefore, for any $i \neq j > B$
\[
\E[P_{ij}] \geq 2 \int_0^\infty \frac{\phi(0)^2}{\phi(t)^2} dt - \underset{a \to \infty}{o(1)}\;.
\]

\paragraph{Only the size of one job is known}
For $i \leq B$ and $j > B$, the algorithm knows $x^a_i$, therefore $\uij_{ij}(\cdot)$ can be a function of $x^a_i$. By Lemma \ref{lem:P-uij} then Inequality \eqref{eq:Pxa>Px}, it holds for all $t\geq 0$ that
\begin{align*}
\Pr(P_{ij} > t) 
&= \Pr(x^a_i > \uij_{i,j}(t) \text{ and } x^a_j > t - \uij_{i,j}(t))\\
&= \E[ \indic{x^a_i > \uij_{i,j}(t)} \indic{x^a_j > t - \uij_{i,j}(t)} ]\\
&= \E\big[ \indic{x^a_i > \uij_{i,j}(t)} \E[\indic{x^a_j > t - \uij_{i,j}(t)} \mid x^a_i,\uij_{i,j}(t)] \big]\\
&= \E\big[ \indic{x^a_i > \uij_{i,j}(t)} \Pr\big(x^a_j > t - \uij_{i,j}(t) \mid \uij_{i,j}(t)\big) \big]\\
&\geq \E\left[ \indic{x^a_i > \uij_{i,j}(t)} \frac{\phi(0) \indic{t - \uij_{i,j}(t) < a}}{\phi(t-\uij_{i,j}(t))} \right]\\
&\geq \phi(0) \E\left[ \frac{\indic{t-a < \uij_{i,j}(t) < x^a_i}}{\phi(t-\uij_{i,j}(t))} \right]\;,
\end{align*}
and we obtain by integrating over $t$ and then using Tonelli's theorem that
\begin{align*}
\E[P_{ij}]
&= \int_0^\infty \Pr(P_{ij} > t) dt\\
&\geq \phi(0) \E\left[ \int_0^\infty \frac{\indic{t-a < \uij_{i,j}(t) < x^a_i}}{\phi(t-\uij_{i,j}(t))} dt \right]\;.
\end{align*}
Recall that $\uij_{i,j}(t)$ is defined as the time spent on job $i$ when a total of $t$ units of time have been spent on both jobs $i$ and $j$, thus $\uij_{ij}(0) = 0$, $\uij_{ij}$ is non-decreasing and it is 1-Lipschitz. Let us denote by $\V$ the set of functions on $[0,\infty)$ satisfying these properties
\[
\V = \Big\{ v : [0,\infty) \to [0,\infty) \; : \; v(0) = 0 \text{ and } \tfrac{v(t_2) - v(t_1)}{t_2 - t_1} \in [0,1] \; \forall t_1 < t_2 \Big\}\;,
\]
and for all $x \geq 0$ and $T \in [0,\infty]$, we denote by $\V_{x,T}$ the subset of $\V$ defined as
\[
\V_{x,T} = \Big\{v \in \V : \inf\{t \geq 0: v(t) \geq x\} = T \Big\}\;.
\]
For any $x \geq 0$ and $v \in \V$, let $T_v = \inf\{t \geq 0: v(t) \geq x\}$, then $v \in \V_{x,T_v}$. Since $v$ is $1$-Lipschitz and $v(0) = 0$, it holds that $T_v \geq x$. If $v(t) < x$ for all $t\geq 0$ then $T_v = \infty$. Therefore, we have for any $x \geq 0$
\[
\V = \bigcup_{T \in [x,\infty]} \V_{x,T}\;.
\]
Furthermore, given that $v$ is non-decreasing, $v(t) < x$ if and only if $t < T_v$.
Consequently, $\E[P_{ij}]$ satisfies
\begin{align}
\E[P_{ij}]
&\geq \phi(0) \E\left[ \inf_{v \in \V} \int_0^\infty \frac{\indic{t-a <v(t) < x^a_i}}{\phi(t-v(t))} dt \right] \nonumber \\
&= \phi(0) \E\left[ \inf_{T \in [x^a_i, \infty]} \inf_{v \in \V_{x^a_i,T}} \int_0^\infty \frac{\indic{t-a <v(t) < x^a_i}}{\phi(t-v(t))} dt \right] \nonumber \\
&= \phi(0) \E\left[ \inf_{T \in [x^a_i, \infty]} \inf_{v \in \V_{x^a_i,T}} \int_0^T \frac{\indic{t-a <v(t)}}{\phi(t-v(t))} dt \right] \;. \label{eq:ineqE-inf-inf}
\end{align}
Let $x > 0$ and $T \in [x,\infty]$, and define $v_{x,T}^*: t \geq 0 \mapsto (t-T+x) \indic{t > T-x}$, then for all $v \in \V_{x,T}$ we have $v(t) \geq v_{x,T}^*(t)$ for all $t \in [0,T]$. Indeed, $v(t) \geq 0 = v_{x,T}^*(t)$ for $t \in [0,T-x]$, and if $v(t) < v_{x,T}^*(t)$ for some $t \in [T-x,T]$ then, because $v$ is $1$-Lipschitz, we have
\[
v(T) \leq v(T-x) + x < v_{x,T}^*(T-x) + x = x\;,
\]
which contradicts $v(T) \geq x$. Finally, as $\phi$ and $v \mapsto \indic{t-a < v}$ are both non-decreasing, it holds for any $v \in \V_{x,T}$ that
\begin{align}
\int_0^T \frac{\indic{t-a < v(t)}}{\phi(t-v(t))}dt
&\geq \int_0^T \frac{\indic{t-a < v_{x,T}^*(t)}}{\phi(t-v_{x,T}^*(t))} dt \nonumber \\
&= \int_0^{T-x} \frac{\indic{t-a < 0}}{\phi(t)} dt + \int_{T-x}^T \frac{\indic{t-a < t - T + x}}{\phi(T-x)} dt \nonumber\\
&= \int_0^{\min(a,T-x)} \frac{dt}{\phi(t)} + \frac{x \indic{T-x < a}}{\phi(T-x)} \nonumber \\
&= \left\{
    \begin{array}{ll}
         \int_0^{T-x} \frac{dt}{\phi(t)} + \frac{x}{\phi(T-x)} & \mbox{if }  T-x < a\\
        \int_0^a \frac{dt}{\phi(t)} & \mbox{if } T-x \geq a
    \end{array}
\right. \\
&= \left\{
    \begin{array}{ll}
         G_\phi(x,T) & \mbox{if }  T-x < a\\
        \int_0^a \frac{dt}{\phi(t)} & \mbox{if } T-x \geq a
    \end{array}
\right.\;. \label{eq:Ga(x,T)}
\end{align}
Taking the infimum over $v \in \V_{x,T}$ then over $T \in [x,\infty]$ in \eqref{eq:Ga(x,T)} gives for any $x \geq 0$ that
\begin{align*}
\inf_{T \in [x,\infty]} \inf_{v \in \V_{x,T}} \int_0^T \frac{\indic{t-a < v(t)}}{\phi(t-v(t))}dt
&\geq \min\left( \inf_{T \in [x,a+x)} G_\phi(x,T), \int_0^a \frac{dt}{\phi(t)} \right)\\
&\geq \min\left( \inf_{T \in [x,\infty]} G_\phi(x,T), \int_0^a \frac{dt}{\phi(t)} \right)\;,
\end{align*}
and substituting into \ref{eq:ineqE-inf-inf} leads to
\begin{align*}
\E[P_{ij}] 
&\geq \phi(0) \E \left[ \min\left( \inf_{T \in [x^a_1,\infty]} G_\phi(x^a_1,T), \int_0^a \frac{dt}{\phi(t)} \right)  \right]\;.
\end{align*}
Let us denote by $f_a$ and $f_\infty$ respectively the density functions of $x_1^a$ and $x_1^\infty$. We have for any $x > 0$ that
\begin{equation}
f_a(x) 
= \frac{d}{dx} \Pr(x^a_1 \leq x)
= \frac{\phi'(x)/\phi(x)^2}{\phi(0)^{-1} - \phi(a)^{-1}} \indic{x < a}
\geq \frac{\phi(0)\phi'(x)}{\phi(x)^2}\indic{x < a}
= f_\infty(x) \indic{x < a}\;,
\end{equation}
therefore
\begin{align}
\E[P_{ij}] 
&= \phi(0) \int_0^\infty \min\left( \inf_{T \in [x,\infty]} G_\phi(x,T), \int_0^a \frac{dt}{\phi(t)} \right) f_a(x) dx \nonumber\\
&\geq \phi(0) \int_0^\infty \min\left( \inf_{T \in [x,\infty]} G_\phi(x,T), \int_0^a \frac{dt}{\phi(t)} \right) f_\infty(x) \indic{x < a} dx\;. \label{eq:EP>int-min-a}
\end{align}
Observing for all $x > 0$ that
\begin{equation}\label{eq:infG<int1/phi}
\inf_{T \in [x,\infty]} G_\phi(x,T)
\leq \lim_{T \to \infty} G_\phi(x,T) 
= \int_0^\infty \frac{dt}{\phi(t)}
\end{equation}
and that the mapping $a \mapsto \min\left( \inf_{T \in [x,\infty]} G_\phi(x,T), \int_0^a \frac{dt}{\phi(t)} \right) f_\infty(x) \indic{x < a}$ is non-decreasing, the monotone convergence theorem, the continuity of the minimum, then Inequality \ref{eq:infG<int1/phi} guarantee that
\begin{align*}
\lim_{a \to \infty} \int_0^\infty \min &\left( \inf_{T \in [x,\infty]} G_\phi(x,T), \int_0^a \frac{dt}{\phi(t)} \right) f_\infty(x) \indic{x < a} dx\\
&= \int_0^\infty \lim_{a \to \infty} \left\{ \min\left( \inf_{T \in [x,\infty]} G_\phi(x,T), \int_0^a \frac{dt}{\phi(t)} \right)f_\infty(x) \indic{x < a} \right\} dx\\
&= \int_0^\infty \min\left( \inf_{T \in [x,\infty]} G_\phi(x,T), \int_0^\infty \frac{dt}{\phi(t)} \right) f_\infty(x)  dx\\
&= \int_0^\infty \left( \inf_{T \in [x,\infty]} G_\phi(x,T) \right) f_\infty(x) dx\\
&= \phi(0) \int_0^\infty \left( \inf_{T \in [x,\infty]} G_\phi(x,T) \right) \frac{\phi'(x)}{\phi(x)^2} dx \;.
\end{align*}
Finally, the previous term is finite because 
\begin{align*}
\int_0^\infty \left( \inf_{T \in [x,\infty]} G_\phi(x,T) \right) \frac{\phi'(x)}{\phi(x)^2} dx
&\leq \int_0^\infty  G_\phi(x,2x) \frac{\phi'(x)}{\phi(x)^2} dx\\
&= \int_0^\infty \left( \int_0^{x} \frac{dt}{\phi(t)} + \frac{x}{\phi(x)}  \right) \frac{\phi'(x)}{\phi(x)^2} dx\\
&\leq  \int_0^\infty \left(2 \int_0^{x} \frac{dt}{\phi(t)} \right) \frac{\phi'(x)}{\phi(x)^2} dx\\
&= 2 \int_0^\infty \left( \int_t^\infty \frac{\phi'(x)}{\phi(t)^2} \right) \frac{dt}{\phi(t)}\\
&= 2 \int_0^\infty \frac{dt}{\phi(t)^2} < \infty\;.
\end{align*}
Therefore 
\begin{align*}
\int_0^\infty \min &\left( \inf_{T \in [x,\infty]} G_\phi(x,T), \int_0^a \frac{dt}{\phi(t)} \right) \frac{\phi(0)}{\phi(t)} \indic{t < a} dt\\
&= \phi(0) \int_0^\infty \left( \inf_{T \in [x,\infty]} G_\phi(x,T) \right) \frac{\phi'(x)}{\phi(x)^2} dx - o(1)\;,    
\end{align*}

and substituting into Inequality \eqref{eq:EP>int-min-a} gives the aimed result.

\end{proof}

\subsubsection{Proof of the theorem}

Using the previous lemmas, we can now prove the theorem.

\begin{proof}
Let $x = (x_1,\ldots,x_n)$ be an array of i.i.d. random job sizes with distribution $\Pr(x_i \leq t) = 1 - \frac{\phi(0)}{\phi(t)}$. Observe that
\[
\E[x_1] 
= \int_0^\infty \Pr(x_1 \geq t) dt
= \phi(0) \int_0^\infty \frac{dt}{\phi(t)}
\quad, \quad
\E[\min(x_1,x_2)] 
= \int_0^\infty \Pr(x_1 \geq t)^2 dt
= \phi(0)^2 \int_0^\infty \frac{dt}{\phi(t)^2}\;.
\]
Lemma \ref{lem:generic-lb} with $a = \infty$ gives for any algorithm $\A$ that $\E[P_{ij}] \geq \E[\min(x_1,x_2)]$ for $i\neq j \leq B$, $\E[P_{ij}] \geq 2 \E[\min(x_1,x_2)]$ for $i\neq j > B$, and $\E[P_{ij}] \geq \alpha_\phi \E[\min(x_1,x_2)]$ for $i\leq B, j > B$. Indeed, the probability density function of $x_1$ is $t \mapsto \frac{\phi(0)\phi'(t)}{\phi(t)^2}$, and  we have for all $i\leq B$ and $j> B$ that
\begin{align*}
\phi(0)\E\big[\inf_{T\geq x_1}G_\phi(x_1,T)\big]
&= \phi(0)^2  \int_0^\infty \left\{\inf_{T\geq x}G_\phi(x,T) \right\} \frac{\phi'(x)}{\phi(x)^2} dx\\
&\geq \alpha_\phi \phi(0)^2 \int_0^\infty \frac{dt}{\phi(t)^2}
= \alpha_\phi \E[\min(x_1,x_2)]\;.    
\end{align*}

Assume that $\int_0^\infty \frac{dt}{\phi(t)} < \infty$, i.e.  $\E[x_1] < \infty$, then by Equation \eqref{eq:generic-output} it holds that
\begin{align*}
\E[\A(x)] 
&\geq n \E[x_1] 
    + \sum_{1 \leq i<j \leq B}\E[\min(x_1,x_2)]
    + \sum_{B<i<j\leq n} 2 \E[\min(x_1,x_2)]
    + \sum_{i=1}^B \sum_{j=B+1}^n \alpha_\phi \E[\min(x_1,x_2)]\\
&=  n \E[x_1]  + \left( \frac{B(B-1)}{2} + (n-B)(n-B-1) + \alpha_\phi B(n-B) \right) \E[\min(x_1,x_2)]\;,
\end{align*}
the terms multiplying $\E[\min(x_1,x_2)]$ can be reordered as follows
\begin{align}
\frac{B(B-1)}{2} &+(n-B)(n-B-1) + \alpha_\phi B(n-B) \nonumber
\\&= \frac{B(B-1)}{2}+(n-B)(n-1 + (\alpha_\phi - 1)B) \nonumber\\
&= \frac{B(B-1)}{2}+n(n-1) + (\alpha_\phi - 1)nB - nB - B((\alpha_\phi - 1)B - 1) \nonumber\\
&= \frac{B(B-1)}{2} + n(n-1) + (\alpha_\phi - 2)nB - (\alpha_\phi - 1)B(B-1) - (\alpha_\phi - 2)B \nonumber\\
&=  n(n-1) - (2-\alpha_\phi)(n-1)B + (\tfrac{3}{2} - \alpha_\phi)B(B-1)\\
&= \frac{n(n-1)}{2}\left( 2 - (4-2\alpha_\phi) \frac{B}{n} + (3 - 2 \alpha_\phi)\frac{B(B-1)}{n(n-1)} \right)\\
&= \frac{n(n-1)}{2}\left( 2 - \frac{B}{n} - (3 - 2 \alpha_\phi)\frac{B}{n}\left(1 - \frac{B(B-1)}{n(n-1)}\right) \right)
\;, \label{eq:reorder-terms}
\end{align}
hence
\begin{align*}
\E[\A(x)]    
= n \E[x_1] + C_{n,B} \frac{n(n-1)}{2} \E[\min(x_1,x_2)] \;.
\end{align*}
On the other hand, by equation \eqref{eq:opt-output}
\[
\E[\OPT(x)] 
= \E\left[ \sum_{i=1}^n x_i + \sum_{i<j} \min(x_i,x_j) \right]
= n \E[x_1] + \frac{n(n-1)}{2} \E[\min(x_1,x_2)]\;,
\]
therefore
\begin{align*}
\frac{\E[\A(x)]}{\E[\OPT(x)]}
&\geq \frac{n \E[x_1] + C_{n,B} \frac{n(n-1)}{2} \E[\min(x_1,x_2)]}{n \E[x_1] + \frac{n(n-1)}{2} \E[\min(x_1,x_2)]}\\
&= C_{n,B} - \frac{C_{n,B} - 1}{1 + \frac{n-1}{2} \frac{\E[\min(x_1,x_2)]}{\E[x_1]}}\;,
\end{align*}
and using Lemma \ref{lem:yao}, this yields
\[
\ratio_{n,B}(\A) \geq C_{n,B} - \frac{C_{n,B} - 1}{1 + \frac{n-1}{2} \frac{\E[\min(x_1,x_2)]}{\E[x_1]}}\;.
\]
If $B_n = wn + o(n)$ for some $w \in [0,1]$, then 
\begin{align*}
\CR_B(\A) 
&\geq \liminf_{n\to \infty} \ratio_{n,B_n}(\A)\\
&\geq 2 - 2(2-\alpha_\phi)w + (3-2\alpha_\phi)w^2\;.    
\end{align*}

Assume now that $\int_0^\infty \frac{dt}{\phi(t)} = \infty$, thus $\E[x_1] = \infty$. Let $a > 0$, and $x^a = (x_1^a,\ldots,x_n^a)$ be i.i.d. job sizes with distribution 
\[
\Pr(x_1^a \leq t)
= 1 - \frac{\phi(t)^{-1} - \phi(a)^{-1}}{\phi(0)^{-1} - \phi(a)^{-1}} \indic{t < a}\;.
\]
Therefore $x_1^a$ has a finite expectation $\E[x_1^a] \leq a$. Denoting by $x_1,x_2$ i.i.d. random variables with distribution $\Pr(x_1 \leq t) = 1 - \frac{\phi(0)}{\phi(t)}$, Lemma \ref{lem:generic-lb} and Equation \eqref{eq:generic-output} give for any algorithm $\A$ that 
\begin{align}
\E[\A(x^a)]
&\geq n \E[x^a_1] 
    + \sum_{1 \leq i<j \leq B}(\E[\min(x_1,x_2)] - \underset{a \to \infty}{o(1)})
    + \sum_{B<i<j\leq n} (2 \E[\min(x_1,x_2)] - \underset{a \to \infty}{o(1)}) \nonumber\\
&\quad 
    + \sum_{i=1}^B \sum_{j=B+1}^n (\alpha_\phi \E[\min(x_1,x_2)] - \underset{a \to \infty}{o(1)})\nonumber\\
&= n \E[x^a_1] + C_{n,B} \frac{n(n-1)}{2} \E[\min(x_1,x_2)] - n^2\underset{a \to \infty}{o(1)}\;, \label{eq:lb-A-a}
\end{align}
where the last equation is obtained with the same computation as in the case $\int_0^\infty \frac{dt}{\phi(t)} < \infty$. Moreover, if $a \geq 1$ then
\[
\E[\min(x^a_1,x^a_2)]
= \int_0^\infty \Pr( \min(x^a_1,x^a_2) \geq t ) dt
 = \int_0^\infty  \left(\frac{\phi(t)^{-1} - \phi(a)^{-1}}{\phi(0)^{-1} - \phi(a)^{-1}}\right)^2 \indic{t < a}dt\;,
\]
and we have for any $t \geq 0$ that $\left(\frac{\phi(t)^{-1} - \phi(a)^{-1}}{\phi(0)^{-1} - \phi(a)^{-1}}\right)^2 \leq \frac{\phi(t)^{-2}}{(\phi(0)^{-1} - \phi(1)^{-1})^2}$, which is an integrable function.  Furthermore, $\phi$ is increasing and $1/\phi^2$ is integrable, hence $\lim\limits_{a \to \infty}\phi(a) = \infty$. Therefore, the dominated convergence theorem gives that
\begin{align*}
\lim_{a \to \infty}\E[\min(x^a_1,x^a_2)]
&= \int_0^\infty \lim_{a \to \infty} \left\{ \left(\frac{\phi(t)^{-1} - \phi(a)^{-1}}{\phi(0)^{-1} - \phi(a)^{-1}}\right)^2 \indic{t < a} \right\}dt  \\
&= \int_0^\infty \frac{\phi(0)^2}{\phi(t)^2} dt\\
&= \E[\min(x_1,x_2)]\;.
\end{align*}
It follows that 
\begin{align*}
\OPT(x^a)
&= \E\left[ \sum_{i=1}^n x_i^a + \sum_{i<j} \min(x_i^a, x_j^a) \right]\\
&= n\E[x_1^a] + \frac{n(n-1)}{2} \E[\min(x_1^a, x_2^a)]\\
&= n\E[x_1^a] + \frac{n(n-1)}{2}\E[\min(x_1, x_2)] + n^2 \underset{a \to \infty}{o(1)}\;.
\end{align*}
Combining this with Inequality \eqref{eq:lb-A-a}, then using Lemma \ref{lem:yao}, we obtain that
\[
\ratio_{n,B}
\geq \frac{n \E[x^a_1] + C_{n,B} \frac{n(n-1)}{2} \E[\min(x_1,x_2)] - n^2\underset{a \to \infty}{o(1)}}{n\E[x_1^a] + \frac{n(n-1)}{2}\E[\min(x_1, x_2)] + n^2 \underset{a \to \infty}{o(1)}}\;,
\]
and if $B = wn + o(n)$ then
\[
\CR_B(\A) 
\geq \liminf_{n \to \infty} \ratio_{n,B}(\A)
\geq \frac{\big(2 - 2(2-\alpha_\phi)w + (3-2\alpha_\phi)w^2\big) \E[\min(x_1,x_2)] - \underset{a \to \infty}{o(1)}}{\E[\min(x_1,x_2)] + \underset{a \to \infty}{o(1)}}\;
\]
and finally, taking the limit $a \to \infty$ yields
\[
\CR_B(\A) \geq 2 - 2(2-\alpha_\phi)w + (3-2\alpha_\phi)w^2\;.
\]
\end{proof}

\subsection{Proof of Corollary \ref{cor:lb-exp}}
\begin{proof}
Let us consider i.i.d. exponentially distributed job sizes. This corresponds to $\phi(t) = e^{t}$ for all $t \geq 0$. It holds that $\phi(0) = 1 > 0$, $\phi$ is increasing, $\frac{\phi'}{\phi} = 1$ is non-increasing, and
\[
\int_0^\infty \frac{dt}{\phi(t)} = \int_0^\infty e^{-t} dt = 1
\;, \qquad
\int_0^\infty \frac{dt}{\phi(t)^2} = \int_0^\infty e^{-2t} dt = \frac{1}{2}\;.
\]
Furthermore, it holds for all $x \geq 0$ and $T \geq x$ that
\begin{align*}
G(x,T)
&= \int_0^{T-x} \frac{dt}{\phi(t)} + \frac{x}{\phi(T-x)}\\
&= \int_0^{T-x} e^{-t} dt + x e^{-(T-x)}\\
&= 1 + (x-1)e^{-(T-x)}\;.
\end{align*}
Thus, for all $x \geq 0$
\[
\inf_{T\geq x} G(x,T) = \left\{
    \begin{array}{ll}
        G(x,x) = x & \mbox{if } x < 1 \\
        \lim\limits_{T \to \infty} G(x,T) = 1  & \mbox{if } x \geq 1
    \end{array}
\right. \;,
\]
and it follows that
\begin{align*}
\phi(0)^2 \int_0^\infty \left\{ \inf_{T \geq x} G_\phi(x,T) \right\} \frac{\phi'(x)}{\phi(x)^2} dx
&= \int_0^\infty \left\{ \inf_{T \geq x} G_\phi(x,T) \right\} e^{-x} dx\\
&= \int_0^1 x e^{-x} dx + \int_1^\infty e^{-x} dx\\
&= 1 - \tfrac{1}{e}\\
&= 2 (1 - \tfrac{1}{e}) \int_0^\infty \frac{dt}{\phi(t)^2}\;.
\end{align*}
Finally, by Theorem \ref{thm:lower-bound}, it holds for any algorithm $\A$ having access to the sizes of $B$ jobs that
\[
\ratio_{n,B}(\A)
\geq C_{n,B} - \frac{C_{n,B}-1}{1 + \frac{n-1}{4}}
= C_{n,B} - \frac{4(C_{n,B}-1)}{n+3}\;,
\]
with 
\[
C_{n,B} = 2 - (4/e)\frac{B}{n} + (4/e - 1)\frac{B(B-1)}{n(n-1)}\;.
\]
In particular, if $B = wn + o(n)$, then again by Theorem \ref{thm:lower-bound} we have
\[
\CR(\A)
\geq 2 - \tfrac{4}{e}w + (\tfrac{4}{e} - 1)w^2\;.
\]
\end{proof}

\subsection{Proof of Corollary \ref{cor:lb-phir}}
\begin{proof}
Let $r\in (\tfrac{1}{2},1)$, and consider $\phi: t \mapsto (1+t)^r$. It holds that $\phi(0) = 1 > 0$, $\phi$ is increasing, $\frac{\phi'}{\phi}: t \mapsto \frac{r}{1+t}$ is non-increasing, and 
$\int_0^\infty \frac{dt}{\phi(t)^2} 
= \int_0^\infty \frac{dt}{(1+t)^{2r}}
= \frac{1}{2r - 1} < \infty$ because $r > \frac{1}{2}$. Moreover, it holds for all $x\geq 0$ and $T \geq x$ that
\begin{align*}
G(x,T)
&= \int_0^{T-x} \frac{dt}{\phi(t)} + \frac{x}{\phi(T-x)}\\
&= \int_0^{T-x} \frac{dt}{(1+t)^r} + \frac{x}{(1+T-x)^r}\\
&= - \frac{1}{1-r} + \frac{(1+T-x)^{1-r}}{1-r} + \frac{x}{(1+T-x)^r}\;,
\end{align*}
therefore, for all $x \geq 0$
\begin{align*}
\inf_{T \geq x} G(x,T)
&=  - \frac{1}{1-r} + \inf_{T \geq x}\left\{ \frac{(1+T-x)^{1-r}}{1-r} + \frac{x}{(1+T-x)^r} \right\}\\
&= - \frac{1}{1-r} + \inf_{y \in [0,1]} \left\{ \frac{y^{-(1-r)}}{1-r} + x y^r \right\}\;,
\end{align*}
where the last inequality is obtained by considering $y = \frac{1}{1+T-x}$. It holds that
\begin{align*}
\frac{d}{dy}\left( \frac{y^{-(1-r)}}{1-r} + x y^r \right)
&= - y^{-(2-r)} + rxy^{-(1-r)}\\
&= rxy^{-(2-r)} \left(y - \frac{1}{rx} \right)\;,
\end{align*}
The mapping $y \mapsto \frac{y^{-(1-r)}}{1-r} + x y^r$ is thus minimized on $[0,\infty)$ for $y = \frac{1}{rx}$, and it is minimized on $[0,1]$ for $y_* = \min(1,\frac{1}{rx})$. Therefore, it holds for $x \leq \frac{1}{r}$ that $y_* = 1$ and 
$\inf_{T \geq x} G(x,T) 
=  - \frac{1}{1-r} + (\frac{1}{1-r} + x) 
= x$, and for $x > \frac{1}{r}$ we have $y_* = \frac{1}{rx}$ and 
\begin{align*}
\inf_{T \geq x} G(x,T) 
&= -\frac{1}{1-r} + \frac{(rx)^{1-r}}{1-r} + (rx)^{-r}\\
&= -\frac{1}{1-r} + \left( \frac{1}{1-r} + \frac{1}{r} \right) (rx)^{1-r}\\
&= -\frac{1}{1-r} + \frac{(rx)^{1-r}}{r(1-r)}\;.
\end{align*}
We deduce that
\begin{align}
\phi(0)^2 \int_0^\infty \left\{ \inf_{T \geq x} G_\phi(x,T) \right\} \frac{\phi'(x)}{\phi(x)^2} dx
&= \int_0^{1/r} \left( -\frac{1}{1-r} + x \right) \frac{r}{(1+x)^{r+1}} dx \nonumber\\
&\quad +  \int_{1/r}^\infty \left( -\frac{1}{1-r} + \frac{(rx)^{1-r}}{r(1-r)} \right) \frac{r}{(1+x)^{r+1}} dx \nonumber\\
&= - \frac{1}{1-r} + \int_0^{1/r} \frac{rx}{(1+x)^{1+r}} dx + \frac{r^{1-r}}{1-r}  \int_{1/r}^\infty \frac{x^{1-r}}{(1+x)^{r+1}}dx \nonumber\\
&\geq - \frac{1}{1-r} + \frac{r^{1-r}}{1-r} \int_{1/r}^\infty \frac{x^{1-r}}{(1+x)^{r+1}}dx\;. \label{eq:phi-r}
\end{align}
Let us denote $W_r = \int_{1/r}^\infty \frac{x^{1-r}}{(1+x)^{r+1}}dx$. It holds that
\begin{flalign*}
\hspace{3cm}&&W_r
=& \int_{1/r}^\infty \frac{x^{1-r}}{(1+x)^{r+1}}dx\\
&&=& \int_{1/r}^\infty \left( \frac{x}{1+x} \right)^{1-r} \frac{dx}{(1+x)^{2r}}\\
&&\geq& \int_{2}^\infty \sqrt{\frac{x}{1+x}} \cdot \frac{dx}{(1+x)^{2r}} &&& (r>\tfrac{1}{2})\\
&&\geq& \int_{0}^{1/3} \frac{\sqrt{1 - s}}{s^{2 - 2r}}ds &&& (s \gets \tfrac{1}{1+x})\\
&&=& \int_{0}^{1/3} \sqrt{1 - s} \cdot \frac{d}{ds}\left( \frac{s^{2r-1}}{2r - 1} \right)ds\\
&&=& \left[ \frac{s^{2r-1}\sqrt{1-s}}{2r-1} \right]_0^{1/3} + \frac{1}{2(2r-1)}\int_0^{1/3} \frac{s^{2r-1}}{\sqrt{1-s}} ds\\
&&=& \frac{3^{-(2r-1)}\sqrt{2/3}}{2r-1}  + \frac{1}{2(2r-1)}\int_0^{1/3} \frac{s^{2r-1}}{\sqrt{1-s}} ds\;.
\end{flalign*}
It holds for all $r \in (\tfrac{1}{2},1)$ that $\frac{s^{2r-1}}{\sqrt{1-s}} \leq \frac{1}{\sqrt{1-s}}$, which is an integrable function on $[0,\tfrac{1}{3}]$, hence the dominated convergence theorem guarantees that 
\begin{align*}
\lim_{r \to \tfrac{1}{2}} \frac{1}{2} \int_0^{1/3} \frac{s^{2r-1}}{\sqrt{1-s}} ds
&= \frac{1}{2} \int_0^{1/3} \lim_{r \to \tfrac{1}{2}} \frac{s^{2r-1}}{\sqrt{1-s}} ds\\
&= \frac{1}{2}\int_0^{1/3} \frac{ds}{\sqrt{1-s}}\\
&= \left[ -\sqrt{1-s} \right]_0^{1/3}\\
&= 1 - \sqrt{2/3}\;,
\end{align*}
we deduce that
\begin{align*}
(2r-1) W_r 
&= 3^{-(2r-1)}\sqrt{2/3} + \frac{1}{2}\int_0^{1/3} \frac{s^{2r-1}}{\sqrt{1-s}} ds\\
&= \left(\sqrt{2/3} - \underset{r \to 1/2}{o(1)} \right) + \left(1- \sqrt{2/3} - \underset{r \to 1/2}{o(1)} \right)\\
&= 1 + \underset{r \to 1/2}{o(1)}\;.
\end{align*}
Substituting into Inequality \ref{eq:phi-r}, and recalling that $\int_0^\infty \frac{dt}{\phi(t)^2} = \frac{1}{2r-1}$, we deduce that for $r$ close to $\frac{1}{2}$
\begin{align*}
\phi(0)^2 \int_0^\infty \left\{ \inf_{T \geq x} G_\phi(x,T) \right\} \frac{\phi'(x)}{\phi(x)^2} dx
&\geq - \frac{1}{1-r} + \frac{r^{1-r}}{1-r} W_r\\
&= \left(- \frac{2r-1}{1-r} + \left(\frac{r^{1-r}}{1-r}\right)(2r-1) W_r \right) \int_0^\infty \frac{dt}{\phi(t)^2}\\
&= \left( o(1) + ( \sqrt{2} + o(1))( 1 - o(1)) \right) \int_0^\infty \frac{dt}{\phi(t)^2}\\
&= \left( \sqrt{2} + o(1) \right) \int_0^\infty \frac{dt}{\phi(t)^2}\;.
\end{align*}
Consequently, by Theorem \ref{thm:lower-bound}, if $B = wn + o(n)$, it holds for any algorithm $\A$ that
\[
\CR(\A) \geq 2 - 2\left( 2 - \sqrt{2} - \underset{r \to 1/2}{o(1)}\right)w + \left(3 - 2 \sqrt{2} - \underset{r \to 1/2}{o(1)} \right)w^2\;,
\]
and taking the limit when $r \to \frac{1}{2}$ gives
\[
\CR(\A) \geq 2 - 2( 2 - \sqrt{2})w + (3 - 2 \sqrt{2})w^2\;.
\]
\end{proof}

\section{Known partial order}\label{appx:orderonly}

\subsection{Preliminary results}

\begin{lemma}\label{lem:Emin}
For any real numbers $x_1,\ldots,x_n$, if $\sigma$ is a uniformly random permutation of $[n]$ then
\[
\E[\min(x_{\sigma(1)}, x_{\sigma(2)})]
= \frac{2}{n(n-1)} \sum_{1\leq i<j \leq n} \min(x_{i}, x_{j}) \;.
\]
\end{lemma}
\begin{proof}
Since $\sigma$ is uniformly random, we have for any $i\neq j$ that $\E[\min(x_{\sigma(1)}, x_{\sigma(2)})] = \E[\min(x_{\sigma(i)}, x_{\sigma()})]$, thus
\begin{align*}
\E[\min(x_{\sigma(1)}, x_{\sigma(2)})]
&= \frac{2}{n(n-1)} \E\Big[\sum_{1\leq i<j \leq n} \min(x_{\sigma(i)}, x_{\sigma(j)})\Big]\\
&= \frac{2}{n(n-1)} \sum_{1\leq i<j \leq n} \min(x_{i}, x_{j}) \;.
\end{align*}
\end{proof}

\begin{lemma}\label{lem:Emin/2}
For any real numbers $x_1,\ldots,x_n$, if $\sigma$ is a uniformly random permutation of $[n]$ then
\[
\E[x_{\sigma(1)} \indic{x_{\sigma(1)} < x_{\sigma(2)}}]
\leq \E[\min(x_{\sigma(1)}, x_{\sigma(2)})]
= \frac{1}{n(n-1)} \sum_{1\leq i<j \leq n} \min(x_{i}, x_{j}) \;,
\]
with equality if $x_1, \ldots, x_n$ are pairwise distinct.
\end{lemma}
\begin{proof}
$\sigma$ is a uniformly random permutation of $[n]$, therefore $\E[x_{\sigma(i)} \indic{x_{\sigma(i)} < x_{\sigma(j)}}] = \E[x_{\sigma(1)} \indic{x_{\sigma(1)} < x_{\sigma(2)}}]$ for all $i \neq j \in [n]$. It follows that
\begin{align*}
\E[x_{\sigma(1)} \indic{x_{\sigma(1)} < x_{\sigma(2)}}]
&= \frac{1}{2} \left( \E[x_{\sigma(1)} \indic{x_{\sigma(1)} < x_{\sigma(2)}}] + \E[x_{\sigma(2)} \indic{x_{\sigma(2)} < x_{\sigma(1)}}] \right)\\
&\leq \frac{1}{2}\E[x_{\sigma(i)} \indic{x_{\sigma(1)} < x_{\sigma(2)}} + x_{\sigma(2)} \indic{x_{\sigma(2)} \leq x_{\sigma(1)}}]\\
&= \frac{1}{2} \E[\min(x_{\sigma(1)}, x_{\sigma(2)})]\;,
\end{align*}
and using Lemma \ref{lem:Emin} concludes the proof.
To have equality, it must hold that $\indic{x_{\sigma(2)} < x_{\sigma(1)}} = \indic{x_{\sigma(2)} \leq x_{\sigma(1)}}$ for any permutation, and this is true if and only if the values $x_1,\ldots,x_n$ are pairwise distinct.
\end{proof}

\subsection{Proof of Theorem \ref{thm:orderonly}}

\begin{proof}
For any $i \neq j \geq B+1$, as in round-robin, $P_{\sigma(i) \sigma(j)} = 2 \min(x_{\sigma(i)}, x_{\sigma(j)})$. For $i \neq j \leq B$, since the jobs $\{x_{\sigma(k)}\}_{k=1}^B$ are run in non-decreasing order one after the other, it holds that $P_{\sigma(i) \sigma(j)} = \min(x_{\sigma(i)}, x_{\sigma(j)})$. Finally, for $i\leq B$ and $j \geq B+1$, the delay caused by $\sigma(j)$ to $\pi(i)$ is always $D_{\sigma(j) \pi(i)} = \min(x_{\pi(i)}, x_{\sigma(j)})$. On the other hand if $x_{\sigma(j)} \leq x_{\pi(i-1)}$ then $x_{\sigma(j)}$ terminates before phase $i$ begins, thus job $\pi(i)$ does not delay job $\sigma(j)$: $D_{\pi(i) \sigma(j)} = 0$. If $x_{\sigma(j)} > x_{\pi(i-1)}$, then after the first step of phase $i$, the time spent on each of the jobs $\pi(i)$ and $\sigma(j)$ is $x_{\pi(i-1)}$. Both jobs are then executed with identical processing powers until one of them is terminated, thus 
\begin{align*}
D_{\pi(i) \sigma(j)} 
&= \min(x_{\pi(i)}, x_{\sigma(j)}) \indic{x_{\sigma(j)} > x_{\pi(i-1)}} \\
&= x_{\pi(i)} \indic{x_{\sigma(j)} > x_{\pi(i)}} + x_{\sigma(j)}\indic{x_{\pi(i)} \geq x_{\sigma(j)} > x_{\pi(i-1)}}\;.
\end{align*}
With the convention $x_{\pi(0) = 0}$, taking the sum over $i$ gives
\begin{align}
\sum_{i=1}^B D_{\pi(i) \sigma(j)} 
&= \sum_{i=1}^B x_{\pi(i)} \indic{x_{\sigma(j)} > x_{\pi(i)}}
 + x_{\sigma(j)} \sum_{i=1}^B \indic{x_{\pi(i)} \geq x_{\sigma(j)} > x_{\pi(i-1)}} \nonumber\\
&= \sum_{i=1}^B x_{\sigma(i)} \indic{x_{\sigma(j)} > x_{\sigma(i)}} + x_{\sigma(j)} \indic{x_{\sigma(j)} \leq x_{\pi(B)}} \nonumber \\
&\leq \sum_{i=1}^B x_{\sigma(i)} \indic{x_{\sigma(j)} > x_{\sigma(i)}} + x_{\sigma(j)}\;. \label{ineq:1-ineq-order}
\end{align}
Using Lemma \ref{lem:Emin/2}, and recalling that $\{\pi(i)\}_{i=1}^B = \{\sigma(i)\}_{i=1}^B$, we have in expectation,
\begin{align}
\E\Big[\sum_{i=1}^B D_{\sigma(i) \sigma(j)} \Big] 
&= \E\Big[\sum_{i=1}^B D_{\pi(i) \sigma(j)} \Big]  \nonumber\\
&\leq \frac{B}{2} \E[\min(x_{\sigma(1)}, x_{\sigma(2)})] + \frac{1}{n} \sum_{k=1}^n x_k\;, \label{ineq:2-ineq-order}
\end{align}
and it follows that, for any $j \geq B+1$
\begin{align*}
\E\Big[\sum_{i=1}^B P_{\sigma(i) \sigma(j)} \Big] 
&= \E\Big[\sum_{i=1}^B D_{\sigma(j) \sigma(i)} \Big] + \E\Big[\sum_{i=1}^B D_{\sigma(i) \sigma(j)} \Big]\\
&\leq B \E[\min(x_{\sigma(1)}, x_{\sigma(2)})] + \frac{B}{2} \E[\min(x_{\sigma(1)}, x_{\sigma(2)})] + \frac{1}{n} \sum_{k=1}^n x_k\\
&= \frac{3B}{2} \E[\min(x_{\sigma(1)}, x_{\sigma(2)})] + \frac{1}{n} \sum_{k=1}^n x_k\;,
\end{align*}
hence
\[
\sum_{j=B+1}^n \E\Big[ \sum_{i=1}^B P_{\sigma(i) \sigma(j)} \Big] \leq \frac{3}{2} B(n-B)\E[\min(x_{\sigma(1)}, x_{\sigma(2)})] + \left(1 - \frac{B}{n}\right) \sum_{k=1}^n x_k\;.
\]
Finally, using Equation \eqref{eq:generic-output}, the previous upper bounds on $P_{\sigma(i) \sigma(j)}$ for all $i\neq j$, then Equation \eqref{eq:reorder-terms} with $3/2$ instead of $\alpha_\pi$ gives that the objective of $\CRRR$ on the instance $x = \{x_1,\ldots,x_n\}$ satisfies in expectation 
\begin{align}
\E[\CRRR(x)]
&= \sum_{i=1}^n x_i + \E\Big[\sum_{1 \leq i < j \leq n} P_{\sigma(i) \sigma(j)}\Big] \nonumber \\
&= \sum_{i=1}^n x_i + \sum_{1 \leq i < j \leq B} \E[\min(x_{\sigma(i)}, x_{\sigma(j)})] 
+ \sum_{B<i<j\leq n} 2 \E[\min(x_{\sigma(i)}, x_{\sigma(j)})] \nonumber\\
&\quad + \sum_{j=B+1}^n \E\Big[ \sum_{i=1}^B P_{\sigma(i) \sigma(j)} \Big] \label{eq:alg-output-order}\\
&\leq \left(2 - \frac{B}{n}\right) \sum_{i=1}^n x_i
+ \left( \frac{B(B-1)}{2} + (n-B)(n-B-1) + \frac{3}{2} B(n-B) \right) \E[\min(x_{\sigma(1)}, x_{\sigma(2)})] \nonumber\\
&= \left(2 - \frac{B}{n}\right) \sum_{i=1}^n x_i
+ \left(2 - \frac{B}{n} \right) \frac{n(n-1)}{2}\E[\min(x_{\sigma(1)}, x_{\sigma(2)})] \nonumber \\
&= \left(2 - \frac{B}{n}\right) \sum_{i=1}^n x_i
+ \left(2 - \frac{B}{n} \right) \sum_{1 \leq i < j \leq n} \min(x_i, x_j) \label{eq:prf-exp}\\
&= \left(2 - \frac{B}{n} \right) \OPT(x)\;, \label{eq:prf-opt-obj}
\end{align}
where we used \ref{lem:Emin} for \eqref{eq:prf-exp} then Equation \eqref{eq:opt-output} for \eqref{eq:prf-opt-obj}. Therefore, the $(n,B)$-competitive ratio of Algorithm \ref{algo:orderonly} satisfies
\begin{equation}\label{eq:ratio-order-lb}
\ratio_{n,B}(\CRRR) \leq 2 - \frac{B}{n} \;.   
\end{equation}
For proving the lower bound on $\ratio_{n,B}$, let $\eps > 0$ and let us consider job sizes $x^\eps_i = 1 + i \eps$ for all $i \in [n]$. Observe that the only inequalities we used in the analysis are Inequalities \eqref{ineq:1-ineq-order} ($x_{\sigma(j)}\indic{x_{\sigma(j)} \leq x_{\pi(B)}} \leq x_{\sigma(j)}$), and \eqref{ineq:2-ineq-order} given by Lemma \ref{lem:Emin/2}. The second one becomes equality if the job sizes are pairwise distinct, which is satisfied by $\{x^\eps_i\}_{i=1}^n$. As for Inequality \eqref{ineq:1-ineq-order}, since the job sizes $x^\eps$ are pairwise distinct and all larger than $1$, we can give instead a lower bound on $\E[x_{\sigma(j)} \indic{x_{\sigma(j)} \leq x_{\pi(B)}}]$ for $j > B$ as follows
\begin{align*}
\E[x_{\sigma(j)} \indic{x_{\sigma(j)} \leq x_{\pi(B)}} ]
&= \E[ x_{\sigma(j)} \indic{\sigma(j) \leq \pi(B)}] \\
&\geq \Pr(\sigma(j) < \pi(B))
= \Pr(\sigma(n) < \max_{i\leq B} \sigma(i))\\
&= \sum_{k=B+1}^n \sum_{m = 1}^{k-1} \Pr( \sigma(n) = m , \max_{i\leq B} \sigma(i) = k)\\
&= \sum_{k=B+1}^n \sum_{m = 1}^{k-1} \sum_{\ell = 1}^B \Pr( \sigma(n) = m , \sigma(\ell) = k, \forall i \in [B]\setminus \{\ell\}: \sigma(i) \in [k-1]\setminus \{ m \} )\\
&= \sum_{k=B+1}^n \sum_{m = 1}^{k-1} \sum_{\ell = 1}^B \frac{1}{n!} {k-2 \choose B-1} (B-1)! (n-B-1)!\\
&= \sum_{k=B+1}^n \sum_{m = 1}^{k-1} \frac{B!(n-B-1)!}{n!} {k-2 \choose B-1}\\
&= \frac{B!(n-B-1)!}{n!} \sum_{k=B}^n (k-1){k-2 \choose B-1}\\
&= \frac{B \cdot B!(n-B-1)!}{n!} \sum_{k=B}^n {k \choose B}\\
&= \frac{B \cdot B!(n-B-1)!}{n!} {n \choose B+1}\\
&= \frac{B}{B+1}\;.
\end{align*}
Therefore, instead of \eqref{ineq:1-ineq-order}, we have the lower bound
\begin{align}
\E\Big[\sum_{i=1}^B D_{\sigma(i) \sigma(j)} \Big] 
&= \E\Big[\sum_{i=1}^B D_{\pi(i) \sigma(j)} \Big]  \nonumber\\
&= \frac{B}{2} \E[\min(x_{\sigma(1)}, x_{\sigma(2)})] + \E[x_{\sigma(j)} \indic{x_{\sigma(j)} \leq x_{\pi(B)}}]\\
&\geq \frac{B}{2} + \frac{B}{B+1}\;,
\end{align}
and thus, since we proved earlier that $D_{\sigma(i) \sigma(j)} = \min(x^\eps_{\sigma(i), x^\eps_{\sigma(j)}}$ for all $i \leq B$ and $j > B$, it holds that
\[
\E\Big[\sum_{i=1}^B P_{\sigma(i) \sigma(j)} \Big] 
\geq \sum_{i=1}^B \E[\min(x_{\sigma(i), x_{\sigma(j)}}] + \frac{B}{2} + \frac{B}{B+1}
\geq \frac{3B}{2} + \frac{B}{B+1}
\]
and using that $x^\eps_i \geq 1$ for all $i \in [n]$, we obtain by substituting the previous inequalities into \eqref{eq:alg-output-order} that
\begin{align*}
\E[\CRRR(x^\eps)]
&\geq n + \frac{B(B-1)}{2} + (n-B)(n-B-1) + (n-B)\left( \frac{B}{2} + \frac{B}{B+1} \right)\\
&= \left(n+(n-B)\frac{B}{B+1} \right) + \left( \frac{B(B-1)}{2} + (n-B)(n-B-1) + \frac{3B}{2}(n-B) \right)\\
&= \left(2n - B - \frac{n-B}{B+1}\right) + \left( n(n-1) - \frac{(n-1)B}{2} \right)\\
&= (n+1)\left(n - \frac{B}{2}\right) - \frac{n-B}{B+1}\;.
\end{align*}
On the other hand, Equation \eqref{eq:opt-output}, along with $x^\eps_i \leq 1 + n \eps$ for all $i \in [n]$, gives
\[
\OPT(x^\eps) 
\leq \sum_{i=1}^n (n-i+1) (1+n\eps) 
=  \frac{n(n+1)}{2}(1+n\eps)\;,
\]
and it follows that
\begin{align*}
\ratio_{n,B}(\CRRR) 
&\geq \frac{\E[\CRRR(x^\eps)]}{\OPT(x^\eps)}\\
&\geq \frac{(n+1)\left(n - \frac{B}{2}\right) - \frac{n-B}{B+1}}{\frac{n(n+1)}{2}(1+n\eps)}\\
&= \frac{1}{1+n\eps} \left( 2 - \frac{B}{n}  
 - \frac{2(1-\frac{B}{n})}{(n+1)(B+1)}\right)\;.
\end{align*}
Taking the limit $\eps \to 0$ and using Inequality \eqref{eq:ratio-order-lb}, we obtain that
\[
2 - \frac{B}{n}  
 - \frac{2(1-\frac{B}{n})}{(n+1)(B+1)}
\leq \ratio_{n,B}(\CRRR) 
\leq 2 - \frac{B}{n}\;.
\]

Finally, if $B_n = \lfloor wn \rfloor$ for some $w \in [0,1]$ then $B_n \geq wn$, hence $\CR_B(\CRRR) = \sup_{n \geq 2} R_{n,B_n}(\CRRR) \leq 2 - w$, and this bound is reached by the lower bound on $R_{n,B_n}$ for $n \to \infty$, which gives that $\CR_B(\CRRR) = 2 - w$.
\end{proof}
\section{Predictions of the job sizes}\label{appx:switch}

\subsection{Proof of Theorem \ref{thm:switch-perfect}}

We first compute the mutual delays of the jobs in $\switch$ with any breakpoints. 

\begin{lemma}\label{lem:delays-switch-algo}
For any job sizes $x = (x_1,\ldots,x_n)$, and for any permutation $\sigma$ of $[n]$ and breakpoints $z^\sigma = (z_{\sigma(1)},\ldots,z_{\sigma(B)})$, the Switch algorithm $\switch(z^\sigma,x)$ satisfies for all $i \neq j \in [n]$ that
\begin{alignat*}{3}
P_{\sigma(i) \sigma(j)}
&= x_{\sigma(i)} \indic{z_{\sigma(i)} < z_{\sigma(j)}} + x_{\sigma(j)} \indic{z_{\sigma(i)} > z_{\sigma(j)}} + (\theta_{ij} x_{\sigma(i)} + (1-\theta_{ij})x_{\sigma(j)}) \indic{z_{\sigma(i)} = z_{\sigma(j)}} \qquad
&&\text{if } i, j \leq B\;,\\
P_{\sigma(i) \sigma(j)}
&= 2 \min(x_{\sigma(i)}, x_{\sigma(j)})
&&\text{if } i, j > B\;,\\
P_{\sigma(i) \sigma(j)}
&= x_{\sigma(i)} \indic{x_{\sigma(j)} > z_{\sigma(i)}} + \min(z_{\sigma(i)}, x_{\sigma(j)})
 &&\text{if } i \leq B, j > B\;.
\end{alignat*}
where $\theta_{ij}$ is a Bernoulli random variable with parameter $1/2$ independent of $\sigma$ for all $i \neq j \in [B]$.
\end{lemma}

\begin{proof}
Let us first define the random variables $\theta_{ij}$. For all $i \neq j \in [B]$, if $z_{\sigma(i)} \neq z_{\sigma(j)}$ then let $\theta_{ij}$ be an independent Bernoulli random variable with parameter $1/2$, and if $z_{\sigma(i)} = z_{\sigma(j)}$ then let $\theta_{ij}$ be the indicator that $z_{\sigma(i)}$ comes before $z_{\sigma(j)}$ with the ordering $\pi$. Since $\pi$ is an ordering of the breakpoints chosen uniformly at random, then $\theta_{ij}$ is a Bernoulli random variable with parameter $1/2$ independently of $\sigma$.

For $i \neq j \in [B]$, if $z_{\sigma(i)} < z_{\sigma(j)}$ then job $\sigma(i)$ is executed until completion before job $\sigma(j)$ starts being executed, thus $D_{\sigma(i) \sigma(j)} = x_{\sigma(i)}$ and $D_{\sigma(j) \sigma(i)} = 0$ and $P_{\sigma(i) \sigma(j)} = x_{\sigma(i)}$. By symmetry, if $z_{\sigma(i)} > z_{\sigma(j)}$ then $P_{\sigma(i) \sigma(j)} = x_{\sigma(j)}$. In the case where $z_{\sigma(i)} = z_{\sigma(j)}$, since $\pi$ is chosen uniformly at random among all the permutations of $[B]$ satisfying $z_{\pi(1)} \leq \ldots \leq z_{\pi(B)}$, each of the jobs $\sigma(i), \sigma(j)$ is run until completion before the other one starts with equal probability $1/2$, hence $P_{\sigma(i) \sigma(j)} = \theta_{ij}x_{\sigma(i)} + (1- \theta_{ij})x_{\sigma(j)}$. It follows that
\[
P_{\sigma(i) \sigma(j)}
= x_{\sigma(i)} \indic{z_{\sigma(i)} < z_{\sigma(j)}} + x_{\sigma(j)} \indic{z_{\sigma(i)} > z_{\sigma(j)}} + (\theta_{ij}x_{\sigma(i)} + (1- \theta_{ij})x_{\sigma(j)}) \indic{z_{\sigma(i)} = z_{\sigma(j)}}\;.
\]
For $i \neq j > B$, jobs $\sigma(i), \sigma(j)$ are processed symmetrically and the delays they cause to each other are the same as in round-robin. Therefore $D_{\sigma(i)\sigma(j)} = D_{\sigma(j)\sigma(i)} = \min(x_{\sigma(i)}, x_{\sigma(j)})$, and it holds that
\[
P_{\sigma(i) \sigma(j)}
= 2 \min(x_{\sigma(i)}, x_{\sigma(j)})\;.
\]
For $i \leq B$ and $j > B$, at the time when job $\sigma(i)$ starts being executed, it holds by definition of Algorithm \ref{algo:switch} that either job $\sigma(j)$ is completed or $S_{\sigma(j)} = z_{\sigma(i)}$, hence the delay caused by job $\sigma(j)$ to job $\sigma(i)$ is $D_{\sigma(j) \sigma(i)} = \min(z_{\sigma(i)}, x_{\sigma(j)})$. On the other hand, job $\sigma(i)$ delays job $\sigma(j)$ if and only if $x_{\sigma(j)} > z_{\sigma(i)}$, and in this case job $\sigma(i)$ runs until completion before job $\sigma(j)$ is completed, hence $D_{\sigma(i) \sigma(j)} = x_{\sigma(i)} \indic{x_{\sigma(j)} > z_{\sigma(i)}}$, and it follows that
\[
P_{\sigma(i) \sigma(j)}
= x_{\sigma(i)} \indic{x_{\sigma(j)} > z_{\sigma(i)}} + \min(z_{\sigma(i)}, x_{\sigma(j)})
\]
\end{proof}

\subsubsection{Proof of the theorem}
Using the previous lemma, we now prove Theorem \ref{thm:switch-perfect}

\begin{proof}
Using Lemma \ref{lem:delays-switch-algo}, it holds for all $i \neq j \in [n]$ that $P_{\sigma(i) \sigma(j)} = \min(x_{\sigma(i)}, x_{\sigma(j)})$ if $i,j \leq B$, $P_{\sigma(i) \sigma(j)} = 2 \min(x_{\sigma(i)}, x_{\sigma(j)})$ if $i,j > B$, and if $i \leq B$ and $j > B$ then 
\[
P_{\sigma(i) \sigma(j)} 
= \min(x_{\sigma(i)}, x_{\sigma(j)}) + x_{\sigma(i)} \indic{x_{\sigma(i)} < x_{\sigma(j)}}\;.
\]
Taking the expectation over $\sigma$ then using Lemma \ref{lem:Emin/2} gives
\begin{align}
\E[x_{\sigma(i)} \indic{x_{\sigma(i)} < x_{\sigma(j)}}]
&= \E[x_{\sigma(1)} \indic{x_{\sigma(1)} < x_{\sigma(2)}}] \nonumber\\
&\leq \frac{1}{2} \E[\min(x_{\sigma(1)}, x_{\sigma(2)})]\;, \label{ineq:only-ineq-CR}
\end{align}
hence for any $i \leq B$ and $j > B$
\[
\E[P_{\sigma(i) \sigma(j)}]
\leq \frac{3}{2} \E[\min(x_{\sigma(1)}, x_{\sigma(2)})]\;.
\]
It follows from Equation \eqref{eq:generic-output} that for any instance of $n$ job sizes $x = (x_1,\ldots,x_n)$, the switch algorithm \ref{algo:switch} with breakpoints $(z_{\sigma(i)})_{i=1}^B = (x_{\sigma(i)})_{i=1}^B$ satisfies
\begin{align}
\E[\switch(x^\sigma,x)]
&\leq \sum_{i=1}^n x_i + \sum_{1 \leq i < j \leq B} \E[\min(x_{\sigma(1)}, x_{\sigma(2)})] + \sum_{B < i < j \leq n} 2 \E[\min(x_{\sigma(i)}, x_{\sigma(j)})] \nonumber\\
&\quad+ \sum_{i=1}^B \sum_{j=B+1}^n \frac{3}{2} \E[\min(x_{\sigma(i)}, x_{\sigma(j)})]\nonumber\\
&= \sum_{i=1}^n x_i + \left( \frac{B(B-1)}{2} + (n-B)(n-B-1) + \frac{3}{2} B(n-B) \right)\E[\min(x_{\sigma(1)}, x_{\sigma(2)})]\nonumber\\
&= \sum_{i=1}^n x_i +  \left( 2 - \frac{B}{n} \right) \frac{n(n-1)}{2}\E[\min(x_{\sigma(1)}, x_{\sigma(2)})]\nonumber\\
&=\sum_{i=1}^n x_i + \left( 2 - \frac{B}{n} \right) \sum_{1\leq i<j \leq n} \min(x_{i}, x_{j})\;, \label{eq:align-switchoutcome}
\end{align}
where we used Lemma \ref{lem:Emin} for the last equation. We can assume without loss of generality that $x_1 \leq \ldots \leq x_n$, which yields $\sum_{1\leq i<j \leq n} \min(x_{i}, x_{j}) = \sum_{i=1}^n (n-i) x_i$. Finally, Equation \eqref{eq:opt-output} gives
\begin{align}
\frac{\E[\switch(x^\sigma,x)]}{\OPT(x)}
&\leq \frac{\sum_{i=1}^n x_i + \left( 2 - \frac{B}{n} \right)\sum_{i=1}^n (n-i) x_i}{\sum_{i=1}^n x_i + \sum_{i=1}^n (n-i) x_i}\nonumber \\
&= 2 - \frac{B}{n} - \frac{(1 - \frac{B}{n}) \sum_{i=1}^n x_i}{\sum_{i=1}^n (n-i+1) x_i} \nonumber\\
&= 2 - \frac{B}{n} - \frac{(1 - \frac{B}{n})}{\sum_{i=1}^n (n-i+1) q_i} \;, \label{eq:ratio-yi}
\end{align}
where $q_i = (\sum_{j=1}^n x_j)^{-1} x_i$ for all $i \in [n]$.  The variables $(q_i)_{i=1}^n$ satisfy $\sum_{i=1}^n q_i = 1$, and we can assume without loss of generality that $q_1 \leq \ldots \leq q_n$ (i.e. $x_1 \leq \ldots \leq x_n$). Expression \eqref{eq:ratio-yi} is maximized under these constraints for $q_1 = \ldots = q_n = \frac{1}{n}$. It follows for all job sizes $x_1,\ldots,x_n$ that
\begin{align*}
\frac{\E[\switch(x^\sigma,x)]}{\OPT(x)}
\leq 2 - \frac{B}{n} - \frac{2(1 - \frac{B}{n})}{n+1}\;,
\end{align*}
hence $\ratio_{n,B}(\switch) \leq 2 - \frac{B}{n} - \frac{2(1 - \frac{B}{n})}{n+1}$. Let us now prove that this is exactly the competitive ratio of $\ALG$. Observe that the only inequality we used while analyzing $\switch(x^\sigma,x)$ is Inequality \eqref{ineq:only-ineq-CR}, which becomes an equality if the job sizes are pairwise distinct. Let us, therefore, consider job sizes $x^{(\eps)}_i = 1 + i \eps \in [1,1+n\eps]$ for some $\eps > 0$. All the inequalities becoming inequalities, it holds in particular that
\begin{align*}
\ratio_{n,B}(\switch)
\geq \frac{\E[\switch(x^\sigma,x^{(\eps)})]}{\OPT(x^{(\eps)})}
&= 2 - \frac{B}{n} - \frac{(1 - \frac{B}{n}) \sum_{i=1}^n x^{(\eps)}_i}{\sum_{i=1}^n (n-i+1) x^{(\eps)}_i}\\
&\geq 2 - \frac{B}{n} - \frac{(1 - \frac{B}{n}) \sum_{i=1}^n (1+ n\eps)}{\sum_{i=1}^n (n-i+1)}\\
&= 2 - \frac{B}{n} - (1 + n\eps)\frac{2(1 - \frac{B}{n})}{n+1}\;,\\
\end{align*}
and taking $\eps \to 0$ yields that
\[
\ratio_{n,B}(\switch) 
= 2 - \frac{B}{n} - \frac{2(1 - \frac{B}{n})}{n+1}\;.
\]

Let $w \in [0,1]$. If $B = \lfloor w n \rfloor$, then $B/n \geq w$ and it holds for all $n \geq 2$ that $\ratio_{n, B}(\switch) \leq 2 - w$, and this bound is reached for $n \to \infty$, thus $\CR_B(\switch) = 2-w$.
\end{proof}

\subsection{Proof of Proposition \ref{prop:rand-switch}}

\begin{proof}
It is shown in \cite{motwani1994nonclairvoyant} that for any instance $x = (x_1,\ldots,x_n)$, the expected sum of completion times resulting from a run of $\RTC$ is $\E[\RTC(x)] = \frac{n+1}{2} \sum_{i=1}^n x_i$. Using Equation \ref{eq:align-switchoutcome}, we deduce that the algorithm $\ALG$ that runs $\RTC$ with probability $\frac{2(n-B)}{n(n+3) - 2B}$ and runs $\switch$ with breakpoints $z^\sigma = x^\sigma$ with the remaining probability satisfies
\begin{align*}
\E[\ALG(x)]
&= \frac{2(n-B)}{n(n+3) - 2B} \RTC(x) + \left( 1 - \frac{2(n-B)}{n(n+3) - 2B}\right) \E[\switch(x^\sigma,x)]\\
&= \frac{2(n-B)}{n(n+3) - 2B} \cdot \frac{n+1}{2} \sum_{i=1}^n x_i  + \left( 1 - \frac{2(n-B)}{n(n+3) - 2B}\right)\left( \sum_{i=1}^n x_i + \left( 2 - \frac{B}{n} \right) \sum_{1\leq i<j \leq n} \min(x_{i}, x_{j}) \right)\\
&= \left( \frac{(n-1)(n-B)}{n(n+3) - 2B} +  1\right) \sum_{i=1}^n x_i + \left( 1 - \frac{2(n-B)}{n(n+3) - 2B}\right)\left( 2 - \frac{B}{n} \right) \sum_{1\leq i<j \leq n} \min(x_{i}, x_{j})\\
&= \left( 2 - \frac{B}{n} - \frac{2(1 - \frac{B}{n})(2 - \frac{B}{n})}{n+3 - \frac{2B}{n}} \right) \left( \sum_{i=1}^n x_i +  \sum_{1\leq i<j \leq n} \min(x_{i}, x_{j}) \right)\\
&= \left( 2 - \frac{B}{n} - \frac{2(1 - \frac{B}{n})(2 - \frac{B}{n})}{n+3 - \frac{2B}{n}} \right) \OPT(x)\;.
\end{align*}
\end{proof}

\subsection{Proof of Lemma \ref{lem:alg-output-F}}

\begin{proof}
Let $x_1,\ldots,x_n$ 
Let $\xi$ be a random variable with distribution $F$. Let $i \neq j \leq B$, by symmetry, we can assume that $x_{\sigma(i)} \leq x_{\sigma(j)}$. Lemma \ref{lem:delays-switch-algo} gives that
\begin{align*}
P_{\sigma(i) \sigma(j)}
&= x_{\sigma(i)} \indic{y_{\sigma(i)}<y_{\sigma(j)}} + x_{\sigma(j)} \indic{y_{\sigma(i)}> y_{\sigma(j)}} + \tfrac{x_{\sigma(i)} + x_{\sigma(j)}}{2} \indic{y_{\sigma(i)} = y_{\sigma(j)}} \\
&\leq x_{\sigma(i)} \indic{y_{\sigma(i)}<y_{\sigma(j)}} + x_{\sigma(j)} \indic{y_{\sigma(i)} \geq y_{\sigma(j)}}\\
&= x_{\sigma(i)} + (x_{\sigma(j)} - x_{\sigma(i)})\indic{y_{\sigma(i)}\geq y_{\sigma(j)}}\\
&\leq x_{\sigma(i)} + (x_{\sigma(j)} - y_{\sigma(j)} + y_{\sigma(i)} -x_{\sigma(i)})\indic{y_{\sigma(i)}\geq y_{\sigma(j)}}\\
&\leq x_{\sigma(i)} + ( \eta_{\sigma(i)} + \eta_{\sigma(j)})\\
&= \min(x_{\sigma(i)}, x_{\sigma(j)}) + (\eta_{\sigma(i)} + \eta_{\sigma(j)})\;,
\end{align*}
and we obtain in expectation that for all $i\neq j \leq B$
\begin{equation}\label{eq:anypred-ij<B}
\E[P_{\sigma(i) \sigma(j)}]
\leq \E[\min(x_{\sigma(i)}, x_{\sigma(j)})] + \frac{2}{B}\E[\eta^\sigma]\;.
\end{equation}
For $i \neq j > B$, Lemma \ref{lem:delays-switch-algo} gives
\begin{equation}\label{eq:anypred-ij>B}
\E[P_{\sigma(i) \sigma(j)}]
= 2 \E[\min(x_{\sigma(i)}, x_{\sigma(j)})]\;.
\end{equation}
For $i \leq B$ and $j > B$, we have again by Lemma \ref{lem:delays-switch-algo} that
\[
\E[P_{\sigma(i) \sigma(j)}]
= \E[\min(\xi y_{\sigma(i)}, x_{\sigma(j)})] + \E[x_{\sigma(i)}\indic{\xi y_{\sigma(i)} < x_{\sigma(j)}}]\;.
\]
Conditional on the permutation $\sigma$, i.e. taking the expectation only over $\xi$, the first term can be bounded as follows
\begin{align}
\E[\min(\xi y_{\sigma(i)}, x_{\sigma(j)}) \mid \sigma]
&\leq \E[\min(\xi x_{\sigma(i)}, x_{\sigma(j)}) + \xi \eta_{\sigma(i)}\mid \sigma] \nonumber\\
&= 
    \E[x_{\sigma(j)}\indic{\xi x_{\sigma(i)} \geq x_{\sigma(j)}}
    + \xi x_{\sigma(i)} \indic{\xi x_{\sigma(i)} < x_{\sigma(j)}} \mid \sigma] 
    + \E[\xi] \eta_{\sigma(i)} \nonumber\\
&= 
    x_{\sigma(j)} \Pr(\xi \geq \tfrac{x_{\sigma(j)}}{x_{\sigma(i)}} \mid \sigma)
    + x_{\sigma(i)} \E\Big[\xi \indic{\xi < \tfrac{x_{\sigma(j)}}{ x_{\sigma(i)}}} \mid \sigma \Big]
    + \E[\xi] \eta_{\sigma(i)}\;,  \label{eq:anypred-Dji-i<B<j}
\end{align}
and for the second term, the expectation conditional on $\sigma$ satisfies
\begin{align}
\E[x_{\sigma(i)}\indic{\xi y_{\sigma(i)} < x_{\sigma(j)}} \mid \sigma] 
&\leq \E[y_{\sigma(i)}\indic{\xi y_{\sigma(i)} < x_{\sigma(j)}} + \eta_{\sigma(i)} \mid \sigma ] \nonumber\\
&= y_{\sigma(i)}\Pr(\xi < \tfrac{x_{\sigma(j)}}{y_{\sigma(i)}} \mid \sigma) + \eta_{\sigma(i)} \nonumber\\
&= h_F(x_{\sigma(j)}, y_{\sigma(i)}) +  \eta_{\sigma(i)} \nonumber\\
&\leq h_F(x_{\sigma(j)}, x_{\sigma(i)}) + L_F |x_{\sigma(i)} - y_{\sigma(i)}| + \eta_{\sigma(i)} \label{eq:h-lipschitz}\\
&= x_{\sigma(i)}\Pr(\xi < \tfrac{x_{\sigma(j)}}{x_{\sigma(i)}} \mid \sigma) + (1+L_F) \eta_{\sigma(i)}\;, \label{eq:anypred-Dij-i<B<j}
\end{align}
where we used for \eqref{eq:h-lipschitz} that $h_F$ is $L$-Lipschitz with respect to the second variable. Combining \eqref{eq:anypred-Dji-i<B<j} and \eqref{eq:anypred-Dij-i<B<j} yields 
\begin{align}
\E[P_{\sigma(i) \sigma(j)} \mid \sigma]
&\leq 
    x_{\sigma(j)} + (x_{\sigma(i)} - x_{\sigma(j)}) \Pr(\xi < \tfrac{x_{\sigma(j)}}{x_{\sigma(i)}} \mid \sigma)
    + x_{\sigma(i)} \E\Big[\xi \indic{\xi < \tfrac{x_{\sigma(j)}}{ x_{\sigma(i)}}} \mid \sigma \Big]
    + (1+L_F+\E[\xi]) \eta_{\sigma(i)} \nonumber\\
&= x_{\sigma(j)} + x_{\sigma(i)}\left( 
    \big(1 - \tfrac{x_{\sigma(j)}}{ x_{\sigma(i)}}\big) \Pr( \xi < \tfrac{x_{\sigma(j)}}{ x_{\sigma(i)}} \mid \sigma) + \E\Big[\xi \indic{\xi < \tfrac{x_{\sigma(j)}}{ x_{\sigma(i)}}} \mid \sigma \Big] \right)
    + (1+L_F+\E[\xi]) \eta_{\sigma(i)}
     \nonumber\\
&= x_{\sigma(j)} + x_{\sigma(i)} g_F\big(\tfrac{x_{\sigma(j)}}{x_{\sigma(i)}}\big) +(1+L_F+\E[\xi]) \eta_{\sigma(i)}\;. \label{eq:anypred-Pij}
\end{align}
We can assume without of generality that $x_1 \leq \ldots \leq x_n$. It holds by definition of the constants $\beta_F$ and $\gamma_F$ that $g_F\big(\tfrac{x_{\sigma(j)}}{x_{\sigma(i)}}\big) \leq \beta_F \tfrac{x_{\sigma(j)}}{x_{\sigma(i)}}$ if $\sigma(j) < \sigma(i)$, and $g_F\big(\tfrac{x_{\sigma(j)}}{x_{\sigma(i)}}\big) \leq \gamma_F - \tfrac{x_{\sigma(j)}}{x_{\sigma(i)}}$ if $\sigma(j) > \sigma(i)$, which gives
\begin{align*}
\E\left[x_{\sigma(j)} + x_{\sigma(i)} g_F\big(\tfrac{x_{\sigma(j)}}{x_{\sigma(i)}}\big)\right]
&= \frac{1}{2} \E\left[x_{\sigma(j)} + x_{\sigma(i)} g_F\big(\tfrac{x_{\sigma(j)}}{x_{\sigma(i)}}\big) \mid \sigma(j) < \sigma(i) \right]\\
&\qquad + \frac{1}{2} \E\left[x_{\sigma(j)} + x_{\sigma(i)} g_F\big(\tfrac{x_{\sigma(j)}}{x_{\sigma(i)}}\big) \mid \sigma(j) > \sigma(i)\right]\\
&\leq \frac{1 + \beta_F}{2} \E[x_{\sigma(j)}  \mid \sigma(j) < \sigma(i)] + \frac{\gamma_F}{2} \E[x_\sigma(i)  \mid \sigma(j) > \sigma(i)]\\
&= \frac{1 + \beta_F}{2} \E[\min(x_{\sigma(i)}, x_{\sigma(j)})  \mid \sigma(j) < \sigma(i)] + \frac{\gamma_F}{2} \E[\min(x_{\sigma(i)}, x_{\sigma(j)})  \mid \sigma(j) > \sigma(i)]\\
&= \frac{1 + \beta_F + \gamma_F}{2} \E[\min(x_{\sigma(i)}, x_{\sigma(j)})]\;,
\end{align*}
thus we obtain by taking the expectation over $\sigma$ in \eqref{eq:anypred-Pij} that
\[
\E[P_{\sigma(i) \sigma(j)}]
\leq \Big(\tfrac{1 + \beta_F + \gamma_F}{2} \Big)\E[\min(x_{\sigma(i)}, x_{\sigma(j)})] + \Big(\tfrac{1+L_F+\E[\xi]}{B}\Big) \E[\eta^\sigma]\;.
\]
Using Equation \ref{eq:generic-output}, the inequality above, Inequalities \eqref{eq:anypred-ij<B}, \eqref{eq:anypred-ij>B}, then Equation \ref{eq:reorder-terms} with $\frac{1 + \beta_F + \gamma_F}{2}$ instead of $\alpha_\pi$, we deduce that Algorithm \ref{algo:any-pred} satisfies for any job sizes $x_1 , \ldots , x_n$ that
\begin{align*}
\E[\switch(\xi y^\sigma, x)] 
&= \sum_{i=1}^n x_i +  \sum_{1\leq i<j \leq n} \E[P_{\sigma(i) \sigma(j)}]\\
&\leq \sum_{i=1}^n x_i +  \sum_{1\leq i<j \leq B}\left(\E[\min(x_{\sigma(i)}, x_{\sigma(j)})] + \frac{2}{B} \E[\eta^\sigma] \right)
+ \sum_{B<i<j \leq n} 2 \E[\min(x_{\sigma(i)}, x_{\sigma(j)})]\\
&\qquad + \sum_{i=1}^B \sum_{j=B+1}^n \left(\Big(\tfrac{1 + \beta_F + \gamma_F}{2} \Big)\E[\min(x_{\sigma(i)}, x_{\sigma(j)})] + \Big(\tfrac{1+L_F+\E[\xi]}{B}\Big) \E[\eta^\sigma]\right)\\
&=  \sum_{i=1}^n x_i + \left( \frac{B(B-1)}{2} + (n-B)(n-B-1) + \Big( \tfrac{1 + \beta_F + \gamma_F}{2} \Big)B(n-B) \right) \E[\min(x_{\sigma(1)}, x_{\sigma(2)})]\\
&\qquad + \big( B-1 + (1+L_F+\E[\xi])(n-B) \big)\E[\eta^\sigma]\\
&=  \sum_{i=1}^n x_i + \left( n(n-1) + \Big(1 - \tfrac{\beta_F + \gamma_F}{2} \Big)B(B-1) - \Big( \tfrac{3}{2} - \tfrac{\beta_F + \gamma_F}{2} \Big)B(n-1) \right) \E[\min(x_{\sigma(1)}, x_{\sigma(2)})]\\
&\qquad + \left( B-1 + (1+L_F+\E[\xi])(n-B) \right)\E[\eta^\sigma]\\
&=  \sum_{i=1}^n x_i + \left( 2 - \Big( 3 - \beta_F - \gamma_F \Big)\frac{B}{n} + \Big(2 - \beta_F - \gamma_F \Big)\frac{B(B-1)}{n(n-1)} \right) \sum_{i<j}\min(x_i,x_j)\\
&\qquad + \big( B-1 + (1+L_F+\E[\xi])(n-B) \big)\E[\eta^\sigma]\;,
\end{align*}
where we used Lemma \ref{lem:Emin} for the last equality.
\end{proof}

\subsection{Proof of Lemma \ref{lem:switch-exp}}

\begin{proof}
Let $\xi \sim 1 + \mathcal{E}(1/\rho)$. The mapping $h_F$ defined in Lemma \ref{lem:alg-output-F} becomes for all $s,t > 0$
\[
h_F(s,t) 
= t F(\tfrac{s}{t})
= t \left(1 - e^{-\frac{1}{\rho}(\frac{s}{t}-1)}\right)\indic{t<s}\;.
\]
For all $s>0$, it holds for all $t<s$ that 
\[
\frac{\partial h_F(s,t)}{\partial t} 
= 1 - e^{\frac{1}{\rho}(\frac{s}{t}-1)} - t\times \frac{s}{\rho t^2} e^{-\frac{1}{\rho}(\frac{s}{t}-1)}
= 1 - \left( \frac{s}{\rho t} +1 \right)e^{-\frac{1}{\rho}(\frac{s}{t}-1)}\;,
\]
but the mapping $u \mapsto (\tfrac{u}{\rho} + 1) e^{-\frac{u-1}{\rho}}$ is decreasing on $[1,\infty]$. Indeed
\begin{align*}
\frac{\partial}{\partial u}\left( (\tfrac{u}{\rho} + 1) e^{-\frac{u-1}{\rho}} \right) 
&= \tfrac{1}{\rho} e^{-\frac{u-1}{\rho}} - \tfrac{1}{\rho} (\tfrac{u}{\rho} + 1) e^{-\frac{u-1}{\rho}}
= - \tfrac{u}{\rho^2} e^{-\frac{u-1}{\rho}} < 0\;,
\end{align*}
and since $\frac{s}{t} > 1$ we deduce that
\[
-\frac{1}{\rho} 
= 1 - \lim_{u \to 1} (\tfrac{u}{\rho} + 1) e^{-\frac{u-1}{\rho}}
\leq \frac{\partial h_F(s,t)}{\partial t} 
\leq 1 - \lim_{u \to \infty} (\tfrac{u}{\rho} + 1) e^{-\frac{u-1}{\rho}}
= 1\;,
\]
thus $|\frac{\partial h_F(s,t)}{\partial t}| \leq \max(1,\tfrac{1}{\rho}) = \tfrac{1}{\rho}$ for all $t<s$. Otherwise, if $t\geq s$ then $h_F(s,t) = 0$ and $h_F(s,\cdot)$ is continuous in $t=s$, therefore $t \mapsto h_F(s,t)$ is $\tfrac{1}{\rho}$-Lipschitz.

On the other hand, given that $\xi > 1$ a.s., it holds for  $s \leq 1$ that $\E[\xi \indic{\xi < s}] = 0$, and for $s > 1$ that 
\begin{align*}
\E[\xi \indic{\xi < s}]
&= \Pr(\xi < s) + \E[(\xi-1) \indic{\xi-1 < s-1}]\\
&= F(s) + \int_0^\infty t \indic{t < s-1} \frac{e^{-t/\rho}}{\rho} dt\\
&= F(s) + \int_0^{s-1} \frac{t e^{-t/\rho}}{\rho} dt\\
&= F(s) + \left[ -(t+\rho) e^{-t/\rho} \right]_0^{s-1}\\
&= \left( 1 - e^{-(s-1)/\rho}\right) + \left(-(s-1+\rho)e^{-t/\rho} + \rho \right)\\
&= 1+ \rho - (s+\rho)e^{-(s-1)/\rho}\;.
\end{align*}
If follows for all $s > 0$ that
\begin{align*}
g_F(s) 
&= (1-s)\left(1 -  e^{-(s-1)/\rho}\right)\indic{s >1} + \left(1+ \rho - (s+\rho)e^{-(s-1)/\rho}\right)\indic{s>1}\\
&= \left((2 + \rho) - s - (1+\rho) e^{-(s-1)/\rho}\right)\indic{s>1}\;,
\end{align*}
hence $\beta_F = \sup_{0<s \leq 1} \frac{g_F(s)}{s} = 0$ and 
\[
\gamma_F = \sup_{s \geq 1} (g_F(s) + s)
= \sup_{s \geq 1} \left( (2 + \rho) - (1+\rho) e^{-(s-1)/\rho}  \right)
= 2 + \rho\;.
\]

Finally, with $L_F = \frac{1}{\rho}$, $\beta_F = 0$ and $\gamma_F = 2 + \rho$, and observing that $\E[\xi] = 1 + \rho$, the constants $C^1_{n,B,F}$ and $C^1_{n,B,F}$ defined in Lemma \ref{lem:alg-output-F} satisfy
\begin{align*}
C^1_{n,B,F}
&= 2 - (1 - \rho) \frac{B}{n} - \rho\frac{B(B-1)}{n(n-1)}
= 2 - \frac{B}{n} + \rho \frac{B}{n}\left(1 - \frac{B-1}{n-1} \right)\;,\\
C^2_{n,B,F}
&\leq (2 + \rho + \tfrac{1}{\rho})(n-B) + B
\leq \frac{4}{\rho} (n-B) + B\;,
\end{align*}
and the objective function of Algorithm \ref{algo:any-pred} with any job sizes $x = \{x_1,\ldots,x_n\}$ can be upper-bounded as follows
\begin{align*}
\E[\switch(\xi y^\sigma, x)]
&\leq \sum_{i=1}^n x_i + C^1_{n,B,F}\sum_{i<j}\min(x_i,x_j) + \left( \frac{4}{\rho} (n-B) + B \right)\E[\eta^\sigma]\;,
\end{align*}
and by Equation \eqref{eq:opt-output} we obtain
\begin{align*}
\frac{\E[\switch(\xi y^\sigma, x)]}{\OPT(x)}
&\leq  \frac{\sum_{i=1}^n x_i +  C^1_{n,B,F}\sum_{i<j}\min(x_i,x_j)}{\sum_{i=1}^n x_i + \sum_{i<j}\min(x_i,x_j)}  + \left( \frac{4}{\rho} (n-B) + B \right)\frac{\E[\eta^\sigma]}{\OPT(x)}\\
&= C^1_{n,B,F} - \frac{(C^1_{n,B,F} - 1) \sum_{i=1}^n x_i}{\sum_{i=1}^n x_i +\sum_{i<j}\min(x_i,x_j)} + \left( \frac{4}{\rho} (n-B) + B \right)\frac{\E[\eta^\sigma]}{\OPT(x)}\;.
\end{align*}
As we showed in the proof of Theorem \ref{thm:switch-perfect}, we have $\frac{\sum_{i=1}^n x_i}{\sum_{i=1}^n x_i +\sum_{i<j}\min(x_i,x_j)} \geq \frac{2}{n+1}$, and the minimum is reached when all the job sizes are equal. 
\begin{align*}
\frac{\E[\switch(\xi y^\sigma, x)]}{\OPT(x)}
&\leq C^1_{n,B,F} - \frac{2 (C^1_{n,B,F}-1)}{n+1} + \left( \tfrac{4}{\rho}(1 - \tfrac{B}{n}) + \tfrac{B}{n} \right) \frac{n \E[\eta^\sigma]}{\OPT(x)}\;,
\end{align*}
which yields after taking the supremum over instances $x$ of $n$ jobs
\begin{align}
\ratio_{n,B}(\eta;\switch) 
&\leq C^1_{n,B,F} - \frac{2 (C^1_{n,B,F}-1)}{n+1} + \left( \tfrac{4}{\rho}(1 - \tfrac{B}{n}) + \tfrac{B}{n} \right) \frac{n \E[\eta^\sigma]}{\OPT(x)}\label{ineq:precise-bound-ratio-pred}\\
&\leq C^1_{n,B,F} +  \left( \tfrac{4}{\rho}(1 - \tfrac{B}{n}) + \tfrac{B}{n} \right) \frac{n \E[\eta^\sigma]}{\OPT(x)} \nonumber\\
&= \left( 2 - \tfrac{B}{n} + \rho \tfrac{B}{n}(1 - \tfrac{B-1}{n-1} ) \right) + \left( \tfrac{4}{\rho}(1 - \tfrac{B}{n}) + \tfrac{B}{n} \right) \frac{n \E[\eta^\sigma]}{\OPT(x)}\;.   \nonumber
\end{align}

\end{proof}

\subsection{Proof of Theorem \ref{thm:preferential-algo}}

\begin{proof}
Let $\rho \in (0,1]$, and $x = (x_1,\ldots,x_n)$ be an instance of job sizes all at least equal to $1$. Before starting to run, the algorithm has access to predictions $y_{\sigma(1)}, \ldots, y_{\sigma(B)}$ of $x_{\sigma(1)}, \ldots, x_{\sigma(B)}$, it samples $\xi \sim 1 + \mathcal{E}(1/\rho)$ and sets the breakpoints $z_{\sigma(i)} = \xi y_{\sigma(i)}$ for all $i \in [B]$. With $\sigma$ and the breakpoints fixed, the preferential algorithm runs Algorithm \ref{algo:switch} at rate $\lambda$ and round-robin at rate $1-\lambda$. Lemma \ref{lem:delays-switch-algo} and Equation \ref{eq:generic-output} guarantee that, with a fixed permutation $\sigma$ and fixed breakpoints, Algorithm \ref{algo:switch} is monotonic, as all the mutual delays $P_{\sigma(i) \sigma(j)}$ are non-decreasing functions of $x_{\sigma(i)}$ and $x_{\sigma(j)}$. Following the proof of Lemma 3.1 in \cite{purohit2018improving}, when running the preferential algorithm, the completion times of Algorithm \ref{algo:switch} increase by a factor of $1/\lambda$, while the completion times of round-robin increase by a factor $1/(1-\lambda)$, and since both algorithms are monotonic, the fact that some jobs are run simultaneously by both of them can only improve the objective function. Denoting $\switch(\xi y^\sigma, x)$ the objective function of Algorithm \ref{algo:switch} with instance $x = (x_i)_{i=1}^n$ and breakpoints $\xi y^\sigma = (\xi y_{\sigma(i)})_{i=1}^B$, and $\RR(x)$ the objective function of round-robin with $x$, it follows that
\[
\ALG_\lambda(x) \leq \min\left( \frac{\switch(\xi y^\sigma, x)}{\lambda}, \frac{\RR(x)}{1 - \lambda} \right)\;.
\]
Given that $\RR(x)$ is deterministic and $2$-competitive, and using Lemma \ref{lem:switch-exp}, we obtain by taking the expectation that
\begin{align*}
\E[\ALG_\lambda(x)]
&\leq  \E\left[\min\left( \frac{\RR(x)}{1 - \lambda},  \frac{\switch(\xi y^\sigma, x)}{\lambda} \right) \right] \\
&\leq \min\left(\frac{\RR(x)}{1 - \lambda},  \frac{\E[\switch(\xi y^\sigma, x)]}{\lambda} \right)\\
&\leq \min\left( \frac{2}{1 - \lambda} , \frac{\left( 2 - \tfrac{B}{n} + \rho \tfrac{B}{n}(1 - \tfrac{B-1}{n-1} ) \right) + \left( \tfrac{4}{\rho}(1 - \tfrac{B}{n}) + \tfrac{B}{n} \right) \frac{n \E[\eta^\sigma]}{\OPT(x)}}{\lambda}\right) \OPT(x),
\end{align*}
which concludes the proof.
\end{proof}


\end{document}